\RequirePackage{booktabs}

\documentclass[sn-mathphys-num]{sn-jnl}

\usepackage{graphicx}%
\usepackage{multirow}%
\usepackage{amsmath,amssymb,amsfonts}%
\usepackage{amsthm}%
\usepackage{mathrsfs}%
\usepackage[title]{appendix}%
\usepackage{xcolor}%
\usepackage{textcomp}%
\usepackage{manyfoot}%
\usepackage{booktabs}%
\usepackage{algorithm}%
\usepackage{algorithmicx}%
\usepackage{algpseudocode}%
\usepackage{listings}%
 
\usepackage{soulpos}

\theoremstyle{thmstyleone}%
\newtheorem{theorem}{Theorem}
\newtheorem{proposition}[theorem]{Proposition}%

\theoremstyle{thmstyletwo}%
\newtheorem{remark}{Remark}%

\theoremstyle{thmstylethree}%
\newtheorem{definition}{Definition}%

\usepackage{my_package}

\begin{document}

\title[Article Title]{Low Variance Trust Region Optimization with Independent Actors and Sequential Updates in Cooperative Multi-agent Reinforcement Learning}


\author[1]{\fnm{Bang Giang} \sur{Le}}\email{giangbang@vnu.edu.vn}

\author*[1]{\fnm{Viet Cuong} \sur{Ta}}\email{cuongtv@vnu.edu.vn}

\affil[1]{\orgdiv{Human Machine Interaction Laboratory}, \orgname{VNU University of Engineering and Technology}, \orgaddress{\street{144 Xuan Thuy, Cau Giay}, \postcode{100000}, \city{Hanoi}, \country{Vietnam}}}


\abstract{
Cooperative multi-agent reinforcement learning assumes each agent shares the same reward function and can be trained effectively using the Trust Region framework of single-agent.
Instead of relying on other agents' actions, the independent actors setting considers each agent to act based only on its local information, thus having more flexible applications. 
However, {in the sequential update framework,} it is required to re-estimate the joint advantage function after each individual agent's policy step. 
Despite the practical success of importance sampling, the updated advantage function suffers from exponentially high variance problems, which likely results in unstable convergence.
In this work, we first analyze the high variance advantage both empirically and theoretically.
To overcome this limitation, we introduce a clipping objective to control the upper bounds of the advantage fluctuation in sequential updates. 
With the proposed objective, we provide a monotonic bound with sub-linear convergence to $\varepsilon$-Nash Equilibria.
We further derive two new practical algorithms using our clipping objective.
{The experiment results on three popular multi-agent reinforcement learning benchmarks show that our proposed method outperforms the tested baselines in most environments.}
By carefully analyzing different training settings, our proposed method is highlighted with both stable convergence properties and the desired low advantage variance estimation.
{For reproducibility purposes, our source code is publicity available at}
\url{https://github.com/giangbang/Low-Variance-Trust-Region-MARL}
}




\maketitle

\section{Introduction}
Traditional methods in reinforcement learning (RL) focus on decision-making problems with single-agent environments and achieve impressive success in a wide range of domains \cite{schulman2015trust, schulman2017proximal}.
Recently, there has been growing interest in expanding reinforcement learning algorithms to multi-agent domains, especially in the fully cooperative setting where all agents coordinate to achieve common objectives due to its close connection to the single-agent ones \cite{ kuba2021trust, rashid2020monotonic}.
Multi-agent reinforcement learning (MARL) approaches can be categorized into value-based \cite{rashid2020monotonic} and policy-based \cite{lowe2017multi, de2020independent}.

Most of the practical algorithms for MARL model agents as independent policies \citep{kuba2021trust, rashid2020monotonic, yu2022surprising} that act myopically based on local information about the agents' states and actions while improving the rewards globally. The benefit of independent actors is paramount, including no explicit communication or modeling of behaviors of other agents
and ease of optimization. As a result, a growing interest in learning independent policies has arisen recently \citep{de2020independent,  daskalakis2020independent, leonardos2021global, ding2022independent}. Independent actors with the aid of global information during training time form the popular \textit{Centralized Training with Decentralized Execution (CTDE)} paradigm that prevails in the multi-agent RL literature \cite{foerster2018counterfactual}. 

{In single-agent RL, since the introduction of trust region policy optimization (TRPO) \cite{schulman2015trust}, Trust Region methods are among the most successful frameworks for optimizing the policy due to its well-supported theory and good convergence property.}
{Later, proximal policy optimization (PPO) \cite{schulman2017proximal} is built upon TRPO by introducing a surrogate clipping mechanism to policy updates and achieves competitive performances in various domains \cite{engstrom2019implementation}}.
Inspired by the success of the Trust Region framework for single-agent RL, policy-based approaches attempt to extend the framework into cooperative multi-agent domains. Recently, \cite{yu2022surprising} show that MAPPO, a direct application of the trust region optimization for single-agent PPO, can achieve superior performance on multi-agent benchmarks.
However, its simultaneous update scheme disregards other agents' actions, potentially leading to miscoordinations \cite{kuba2021trust} and non-stationarity problems \cite{hernandez2017survey}. 
Addressing this limitation, \cite{kuba2021trust} propose a theoretically grounded sequential-update framework,
where agents consider the changes of previous agents before updating their policies that takes into account the changes of previous agents with a theoretical guarantee of monotonic improvement. 
Practical algorithms of the sequential update with trust region optimization are Heterogeneous-Agent Trust Region Policy Optimisation (HATPRO) and Heterogeneous-Agent Proximal Policy Optimisation (HAPPO) algorithms, both of which model agents as independent actors.
Recent works further explore the potential of sequential updates beyond independent actors. For example, \cite{wen2022multi} utilizes a Transformer-based architecture as the policy while \cite{zhao2023local} models agents as conditional dependent actors. These approaches highlight the merits of sequential updates, especially when agents can observe and respond to the other agent's actions.

In this work, we revisit the sequential update Trust Region methods with independent actors. 
The advantage estimate of the decomposition, when coupled with independent actors, can potentially suffer from high variance problems that grow exponentially with the number of agents.   
While the sequential update is a proper approach in training heterogeneous multi-agents, it relies on a joint advantage estimator for tackling the credit assignment.
Naive importance sampling of this estimator allows zero bias estimator, but its variance can increase indefinitely.
To address this issue, we propose a novel surrogate objective function that can guarantee a low variance estimate. We further analyze the Trust Region optimization with our objective and show that it converges non-asymptotically to approximate Nash equilibria at a sublinear rate. For practical implementations of our method, we derive two new algorithms clip-HAPPO and clip-HATRPO, and demonstrate their competitive performance on three popular MARL benchmarks. The ablation study further proves that our method reduces the effects of high-variance updates, resulting in more effective learning across the agents. Our contributions are summarized below
\begin{itemize}
    \item We give a systematic analysis on the high variance problem of sequential updates with independent actors in MARL {and demonstrate, both empirically and theoretically, that the variance of the advantage estimate can grow exponentially with the number of agents.}
    \item We propose a novel trust region method that incorporates the clipping mechanism to the advantage estimate {with a low variance guarantee}; We further establish a theoretical analysis of the monotonic bound and the non-asymptotic $O(1/\sqrt{K})$ convergence rate of our algorithm to Nash equilibria. 
    \item {We conduct extensive experiments on three popular MARL benchmarks and the results demonstrate that our method can improve the training progress in general.
    By reducing the estimated advantage variance, our proposed method has better convergence properties and is more stable than other baselines.
    Furthermore, the low variance property makes our Trust Region method able to learn in a variety of settings, which emphasizes the superiority of the clipped mechanism.}
\end{itemize}

{The remaining of this paper is organized as follows: Section \ref{sec:motivation} presents the motivation of the work; our main method is described in Section \ref{sec:method}; the experiments and results are illustrated in the next Section \ref{sec:experiments}; 
we present the related works are presented in Section \ref{sec:relatedwork}; the conclusion is discussed in Section \ref{sec:conclusion}.
Furthermore, the proofs and experiment details are presented in the Appendix Section \ref{appendix:proofs} and Appendix Section \ref{appendix:experimentdetails}, respectively.}

\section{Motivation}
\label{sec:motivation}
\subsection{Background and Preliminaries}
\textbf{Problem formulation.}
We consider the fully cooperative MARL, formulated as a Markov game defined by a tuple $\langle \mathcal N, \mathcal S, {{\mathcal A}}, P, r, \gamma, \rho \rangle$. We denote $\mathcal N = \{1, 2, \dots, N\}$ is the set of all $N$ agents, $\mathcal S$ is the finite state space, ${{\mathcal A}}$ is the joint action space of each $i$ agent's action space $\mathcal A_i$, i.e ${{\mathcal A}} = (\mathcal A_1 \times \dots \times \mathcal A_{N})$, we assume that all action spaces are finite and of equal size, 
$\Pi^i$ is the space of stochastic policies of the agent $i$,
$\rho$ is the initial state distribution,
and $\gamma \in [0, 1)$ is the discount factor. 
Since we consider the fully cooperative setting, all the agents share a single global bounded reward function $r:\mathcal{S}\times \mathbf{\mathcal A} \rightarrow \mathbb R$. 
At each timestep $t$, all agents simultaneously take actions $\mathbf a_t = (a^1_t, a_t^2, \dots, a^{N}_t)$ according to their respective policy $\pi^i \in \Pi^i$ that are \textit{independent} of each other, then the environment transitions to the next state $s_{t+1}$ according to the dynamic kernel $P$. 
Additionally, let $d^\pi_\rho (s) = \mathbb E_{s_0 \sim \rho}(1-\gamma)\sum_{t=0}^\infty \gamma^t \mathbb P(s_t = s)$ be the discounted marginal state distribution induced by the joint policy $\boldsymbol{\pi}$ and $Q_{\boldsymbol{\pi}}(s, \boldsymbol{a})=\mathbb E_{\boldsymbol{\pi}}\big[\sum_{t=0}\gamma^t r_t|s_0=s, \boldsymbol{a}_0=\boldsymbol{a}\big]$ the joint Q function of policy $\boldsymbol{\pi}$. 
The goal of fully cooperative MARL is to maximize the total expected cumulative sum of all the agents' reward, $J(\boldsymbol  {\pi}) = \mathbb E_{\boldsymbol  {\pi}} \sum_t \gamma^t r_t$.

\textbf{Notations. }
Throughout the paper, we use bold notations to denote the joint action $\boldsymbol{a}$ and joint policy $\boldsymbol{\pi}$, while the unbolded letters are used for individual agents. 
We define $\boldsymbol{\pi}^{{1:m}}$ and $\boldsymbol{a}^{{1:m}}$ the joint policy and action of a subset of the first $m$ agents respectively, while $\boldsymbol{\cdot}^{-(1:m)}$ be their complement. 
We also denote $A=|\mathcal A_i|, \forall i\in \mathcal N$ the cardinality of an agent's action space. 
In this paper, we regularly use the notations of the "partially-joint" Q function and advantage function extended to the multi-agent setting, defined below
\begin{definition} \label{def:partialQA}
Let ${1:m}$ be a subset of the first $m$ agents from $\mathcal N$,
and $-{1:m}$ its complement. The multi-agent state-action value function is defined as
\begin{equation*}
    Q_{\boldsymbol{\pi}}^{{1:m}}(s, \boldsymbol{a}^{{1:m}}) \triangleq \mathbb E_{\boldsymbol{a}^{-{1:m}}\sim \boldsymbol{\pi}^{-{1:m}}} \big[ Q_{\boldsymbol{\pi}}\big(   
        s, \boldsymbol{a}^{{1:m}}, \boldsymbol{a}^{-{1:m}}
    \big) \big],
\end{equation*}
and the multiagent advantage function of the $m$-th agent is 
\begin{flalign*}
\begin{multlined}[t]
    A_{\boldsymbol{\pi}}^m\left(s, \boldsymbol{a}^{1:m-1}, {a}^{m}\right) \triangleq Q_{\boldsymbol{\pi}}^{1:m}\left( s, \boldsymbol{a}^{{1:m-1}} , a^m\right) 
    - Q_{\boldsymbol{\pi}}^{1:m-1}\left( s, \boldsymbol{a}^{{1:m-1}} \right).
\end{multlined}
\end{flalign*}
\end{definition}
These definitions are narrower than those provided in \cite{kuba2021trust, zhao2023local}, and are sufficient for our analysis. 

We also denote the regular "fully-joint" advantage function as $A_{\boldsymbol{\pi}}(s, \boldsymbol{a})=Q_{\boldsymbol{\pi}}(s, \boldsymbol{a})-\mathbb E_{\boldsymbol{a} \sim \boldsymbol{\pi}}Q_{\boldsymbol{\pi}}(s, \boldsymbol{a})$. Finally, we omit the arguments of functions when it is clear from the context, for example, $A_{\boldsymbol{\pi}}$ instead of $A_{\boldsymbol{\pi}}(s, \boldsymbol{a})$.


\begin{definition}
    {For any two (individual/joint) policy $\pi$ and $\widetilde{\pi}$ and any state $s \in \mathcal S$, the Kullback–Leibler (KL) divergence between the two policies is defined as}
    \[D_\text{KL}(\pi(\cdot | s)\| \widetilde{\pi}(\cdot|s)) = \sum_a \pi(a|s)\log\left ( \frac{\pi(a|s)}{\widetilde{\pi}(a|s)} \right),\]
    {where $a$ is the individual/joint actions depending on the type of the policies in the context.}
\end{definition}

{\textbf{Trust Region Policy Optimization.} 
TRPO \cite{schulman2015trust} is a single-agent RL algorithm designed to optimize a stochastic policy while ensuring monotonic improvement with each iteration 
by leveraging the lower bound of performance between the new policy $\widetilde\pi$ and the old policy $\pi$:
\begin{equation}
J(\widetilde\pi) - J(\pi) \geq \mathbb E_{s\sim d^{\pi}_\rho, a\sim \widetilde\pi}A_{\pi}(s, a) - C\text{D}_\text{KL}^{\max}(\pi, \widetilde\pi),    
\label{eq:trpo}
\end{equation}
with $C=\frac{4\gamma\max_{s, a}|A_\pi(s, a)|}{(1-\gamma)^2}$ and $\text{D}_\text{KL}^{\max}(\pi, \widetilde\pi)$ is the maximal KL divergence of $\pi$ and $\widetilde\pi$ over all states. TRPO directly transforms the above RHS into a constrained optimization problem, whereas PPO \cite{schulman2017proximal} achieves similar objectives more efficiently by employing a simple clipping trick.
}

\textbf{Nash Policy.} We introduce the solution concept in multi-agent games with the definition of Nash Equilibrium (NE).
\begin{definition}
    In a fully-cooperative game, a joint policy $\boldsymbol{\pi} = (\pi^1, \pi^2, \dots, \pi^N)$ is an $\varepsilon$-Nash equilibrium if for every $i\in \mathcal N, \widetilde{\pi}^i \in \Pi^i$ implies $J(\boldsymbol{\pi})\geq J(\boldsymbol{\pi}^{-i}, \widetilde\pi^i)-\varepsilon$. 
    Every NE is a 0-equilibrium.
\end{definition}


{NE plays an important role in multi-agent games, representing stable equilibriums where no agent benefits from unilaterally deviating from current strategies. 
From the optimization perspective, convergence to NE is the common practice in theoretical multi-agent RL literature \cite{leonardos2021global, ding2022independent}}.


\textbf{Sequential Update Trust Region Methods}.
\cite{kuba2021trust} proposed a trust region method that guarantees monotonic improvement in theory. 
Similar to the TRPO, the main results can be summarized in the following
\begin{lemma} [Kuba et al. \cite{kuba2021trust}, Lemma 2] \label{lem:trustbound}
Let $\boldsymbol{\pi}$ be a joint policy and for any other joint policies $\boldsymbol{\widetilde{\pi}}$ and $\boldsymbol{\hat{\pi}}$, 
define the surrogate loss of multiagent as $L_{\boldsymbol{\pi}}^{1:m}(\boldsymbol{\widetilde{\pi}}^{{1:m-1}}, \hat{\pi}^{m}) \triangleq \mathbb E_{s\sim d^{\boldsymbol{\pi}}_{\rho},a^{{1:m- 1}}\sim\boldsymbol{\widetilde{\pi}}^{{1:m-1}}, a^{m}\sim \hat{\pi}^{m}}[A^{m}_{\boldsymbol{\pi}}(s, \boldsymbol{a}^{{1:m-1}}, a^{m})]$,
then 
\begin{flalign} \label{eq:harlbound}
\begin{multlined}
    J(\boldsymbol{\widetilde{\pi}}) \geq J(\boldsymbol{\pi}) + \frac{1}{1-\gamma}\sum_{m=1}^N \Big[ L^{{1:m}}_{\boldsymbol{\pi}} (\boldsymbol{\widetilde{\pi}}^{{1:m-1}}, \widetilde{\pi}^{m} ) 
    - C
    \mathrm{D}_{\mathrm{KL}}^{\max}
    (\pi^{m}, \widetilde{\pi}^{m})\Big],
\end{multlined}
\end{flalign}
where $C=\frac{4\gamma \epsilon} {1-\gamma}$, $\epsilon=\max_{s, \boldsymbol{a}}|A_{\boldsymbol{\pi}}(s, \boldsymbol{a})|$ and $\mathrm{D}_{\mathrm{KL}}^{\max} (\pi, \widetilde{\pi}) = \max_s {D}_{\mathrm{KL}}\left(\pi(\cdot| s)\| \widetilde{\pi}(\cdot| s)\right)$.
\end{lemma}
Lemma \ref{lem:trustbound} provides a lower bound on the performance of a new joint policy $\boldsymbol{\widetilde{\pi}}$. 
This bound is estimated from the samples collected by the current joint policy $\boldsymbol{\pi}$ and becomes increasingly tight when the two joint policies get closer to each other.
In other words, Lemma \ref{lem:trustbound} is a direct extension of the bound in (\ref{eq:trpo}) into the fully cooperative MARL setting.

An iterative algorithm can be derived where at each iteration $k$, one searches for the next joint policy at the iteration $k+1$ by \textit{sequentially} optimizing the RHS of (\ref{eq:harlbound}) with each local policy $\pi^{m}$
\begin{equation}
    {\pi}_{k+1}^{m} = \argmax_{\pi^{m}}
    \left[ L^{{1:m}}_{\boldsymbol{\pi}_k} (\boldsymbol{{\pi}}_{k+1}^{{1:m-1}}, {\pi}^{m} )- 
    C
    \text{D}_{\text{KL}}^{\max}(\pi_k^{m}, \pi^{m})
    \right] \label{eq:harl_obj}
\end{equation}
The sequential update scheme in (\ref{eq:harl_obj})  monotonically improve after every iteration, i.e. $J(\boldsymbol{\pi}_{k+1}) \geq J(\boldsymbol{\pi}_{k}), \forall k \in \mathbb N$, and converges asymptotically to Nash Equilibria. To solve the optimization problem (\ref{eq:harl_obj}), we need to estimate $L^{{1:m}}_{\boldsymbol{\pi}_k} (\boldsymbol{{\pi}}_{k+1}^{{1:m-1}}, {\pi}^{m} )$ at each update of an agent, or equivalently, we need to estimate the advantage $\mathbb E_{s\sim d^{\boldsymbol{\pi}},a^{{1:m-1}}\sim\boldsymbol{\widetilde{\pi}}^{{1:m-1}}, a^{m}\sim \hat{\pi}^{m}}[A^{m}_{\boldsymbol{\pi}}(s, \boldsymbol{a}^{{1:m-1}}, a^{m})]$.
{Kuba et al.} \cite{kuba2021trust} {estimates this quantity through the following expression}
\begin{multline} \label{eq:advest}
    \mathbb E_{\boldsymbol{a}^{{1:m-1}} \sim \boldsymbol{\widetilde{\pi}}, a^{m}\sim\hat{\pi}^{m}}[A_{\boldsymbol{\pi}}^{m} (s, \boldsymbol{a}^{{i:m-1}}, a^{m})] \\
    = \mathbb E_{\boldsymbol{a} \sim \boldsymbol{\pi}} \Big[  
    \Big(
    \frac{\hat{\pi}^{m}(a^{m}|s)}{\pi^{m}(a^{m}|s)} -1 
    \Big)
    \frac{\boldsymbol{\widetilde{\pi}} ^{{1:m-1}}(\boldsymbol{a}^{{1:m-1}}|s)}{
    \boldsymbol{\pi} ^{{1:m-1}}(\boldsymbol{a}^{{1:m-1}}|s)
    }
    A_{\boldsymbol{\pi}} 
    \Big]. 
\end{multline}
For convenience, we define $\Re^{{1:m-1}}_{\boldsymbol{\widetilde{\pi}}, \boldsymbol{\pi}}(s, \boldsymbol{a}) \triangleq \frac{\boldsymbol{\widetilde{\pi}} ^{{1:m-1}}(\boldsymbol{a}^{{1:m-1}}|s)}{ \boldsymbol{\pi} ^{{1:m-1}}(\boldsymbol{a}^{{1:m-1}}|s)}$ as the shorthand for the ratio between the old and new probabilities after the updates of previous agents, with $\Re^{1:0}_{\boldsymbol{\widetilde{\pi}}, \boldsymbol{\pi}}=1$. The practical implementation of (\ref{eq:advest}) uses Monte-Carlo sampling estimation as a tractable surrogate loss:
\begin{equation} \label{eq:estimate_adv}
    \hat{L}^{{1:m}}_{\boldsymbol{\pi}}(\boldsymbol{\widetilde{\pi}}^{{1:m-1}}, \hat{\pi}^{m}, s, \boldsymbol{a}) =
    \frac{\hat{\pi}^{m}(a^{m}|s)}{\pi^{m}(a^{m}|s)}
    \Re^{{1:m-1}}_{\boldsymbol{\widetilde{\pi}}, \boldsymbol{\pi}} 
    A_{\boldsymbol{\pi}}
    .
\end{equation}


\subsection{Upper bound of the Advantage variance}
The monotonic improvement property in Lemma \ref{lem:trustbound} hinges on the premise that the expected advantage in (\ref{eq:advest}) can be calculated in its closed form. However, its sample approximation involves an importance sampling ratio that, when combined with the assumption that all agents act independently, is the product of importance sampling ratios of individual agents, i.e. $\Re^{{1:m-1}}_{\boldsymbol{\widetilde{\pi}}, \boldsymbol{\pi}}(s, \boldsymbol{a}) = \prod_{k=1}^{m-1} \frac{\widetilde{\pi}^{k}(a^{k}|s)}{\pi^{k}(a^{k}|s)}$ (this is sometimes referred to as the Mean-Field assumption \citep{li2021dealing}).
In this section, we show that {a single agent's advantage estimation} can potentially suffer from high variance which can grow exponentially with the number of agents in worst cases. 
Firstly, we observe that for any state $s \in \mathcal S$ and joint policies $\boldsymbol{\pi}$ and $\boldsymbol{\widetilde\pi}$, the variance of the advantage estimator satisfies 
\begin{align}
        \textbf{Var}_{\boldsymbol{a} \sim\boldsymbol{\pi}}\Re^{{1:m-1}}_{\boldsymbol{\widetilde{\pi}}, \boldsymbol{\pi}}(s, \boldsymbol{a}) 
        A_{\boldsymbol{\pi}}(s, \boldsymbol{a}) &= \mathbb E_{\boldsymbol{a} \sim\boldsymbol{\pi}} \Big[
        \Big(\Re^{{1:m-1}}_{\boldsymbol{\widetilde{\pi}}, \boldsymbol{\pi}}  
        A_{\boldsymbol{\pi}}
        \Big)^2 \Big]
        - 
        \Big(\mathbb E_{\boldsymbol{a} \sim\boldsymbol{\pi}} 
        \Re^{{1:m-1}}_{\boldsymbol{\widetilde{\pi}}, \boldsymbol{\pi}}  
        A_{\boldsymbol{\pi}}
        \Big)^2 \notag \\
        &\leq \epsilon^2
    \mathbb E_{\boldsymbol{a}^{{1:m-1}} \sim \boldsymbol{\widetilde{\pi}}} 
    \big[
    \Re^{{1:m-1}}_{\boldsymbol{\widetilde{\pi}}, \boldsymbol{\pi}}(s, \boldsymbol{a})
    \big] \label{eq:worstcasevar}
\end{align}

As shown in (\ref{eq:worstcasevar}), 
the advantage variance depends on the joint sampling ratio, which is generally difficult to model.
The following illustrates an example where this ratio can grow rapidly.

\begin{restatable}{proposition}{exponentialexample}\label{prop:exponentialexample}
    Consider a fully cooperative environment with one state, an even number of $2N$ agents, and a joint action space $\{0, 1\}^{2N}$. The reward of all agents is 1 when at least $N$ agents take action 1 and -1 otherwise. Let the joint policy $\boldsymbol{\pi}$ of all agents the uniform distribution and $C=3/2$. Then using one sequential update scheme defined in (\ref{eq:harl_obj}) with the loss estimated in (\ref{eq:estimate_adv}), the expectation under the new joint policy $\boldsymbol{\widetilde \pi}$ of the sampling ratio $\Re^{{1:m-1}}_{\boldsymbol{\widetilde{\pi}}, \boldsymbol{\pi}}$ grows exponentially with the number of updated agents
    \begin{equation*}
        \mathbb E_{\boldsymbol{a}^{{1:m-1}} \sim \boldsymbol{\widetilde{\pi}}} 
        \Re^{{1:m-1}}_{\boldsymbol{\widetilde{\pi}}, \boldsymbol{\pi}} \geq \frac{1.5^{m-1}}{2}.
    \end{equation*}
\end{restatable}

Since we are optimizing $\boldsymbol{\widetilde \pi}$ over the advantage $A(s, \boldsymbol{a})$, 
the joint policy is updated to increase the advantage value.
Therefore, the exponential sampling ratio issue can happen if 
each individual policy follows the same direction.
In the proposition \ref{prop:exponentialexample}, we construct a case where the current action probability is increased by a joint positive advantage.
As all the agents use a shared advantage estimate $A_{\boldsymbol{\pi}}(s, \boldsymbol{a})$, 
then they would prefer the same action simultaneously, 
which results in the exponential growth rates of the sampling ratios product. 
This growth rate is not only motivated by theory but can also occur in practice.

{We train HAPPO in two environments \textit{Ant-v2-8x1} and \textit{Humanoid-v2-17x1} of the MaMujoco benchmark \cite{peng2021facmac} and plot the variance of the advantage by agents' order in Figure \ref{fig:exponential_figure}}.
{The benchmark is built upon the continuous control benchmark Mujoco \cite{todorov2012mujoco} with separated joints and each agent controls several joints and attempts to move together.}
{The \textit{Ant-v2-8x1} has 8 agents with position numbers from 0 to 7 and the \textit{Humanoid-v2-17x1} has 17 agents with position numbers from 0 to 16.}
{In each update cycle, the position number corresponds to the policy's update order in the sequential update setting.  For instance, an agent with the number 0 would receive the first policy update and so on.}
{As can be seen from the plot, the growth rate associated with agents' position can increase with the exponential rate in both experiments.}
{In the case of \textit{Humanoid-v2-17x1}, the variance can reach as high as $1e6$ during the training.}
{Such high variance issues likely affect the stability of the joint policy and should be avoided.}
{Note that, the plotted variance is computed along the position of the agents in the update sequence.}
{The practical implementation of HAPPO employs a random update order for each training cycle. Therefore, the effects of high advantage variance are balanced among agents' policies.}




\begin{figure}
\centering\includegraphics[width=.8\columnwidth]{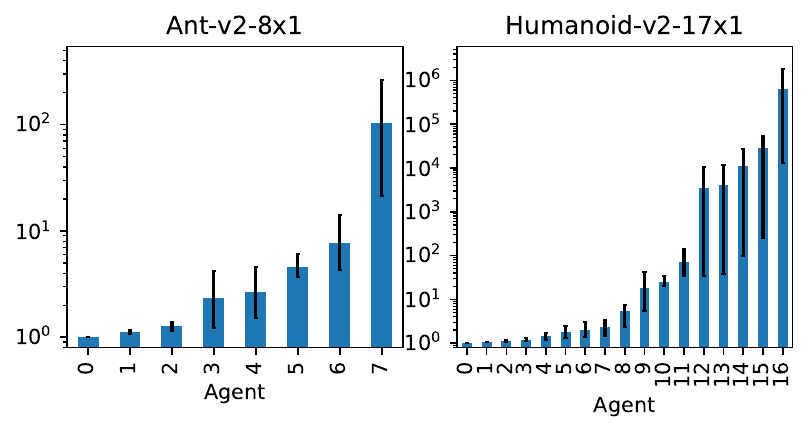}
  \caption{Maximum of the variance of advantage estimate of HAPPO by agents' position in the random permutation setting, showing the exponential growth rate in worst-case as indicated in Proposition \ref{prop:exponentialexample}.
  \label{fig:exponential_figure}}
\end{figure}


\textbf{Can we avoid this explosion of variance with existing methods?} Given that the current theory of trust region method in MARL glosses over this estimation step, it boils down to the question of practical implementations. 
Unfortunately, performing an accurate estimate of the quantity is challenging and incompatible with the on-policy policy gradient framework. 
For example, fully examining all possible actions in (\ref{eq:estimate_adv}) is intractable in continuous action environments, and in discrete ones, it still requires the knowledge of the advantage of arbitrary joint actions $A_{\boldsymbol{\pi}}(s, \boldsymbol{a})$. 
{The problem of exponentially high variance with independent actors is suggested in Coordinated PPO \cite{wu2021coordinated}. However, the proposed double-clipping trick is largely coupled with the PPO implementation, and many of the theoretical properties are not provided.}
Meanwhile, we systematically tackle the high variance problem of general Trust Region methods with the sequential update scheme in MARL, gaining insights into how a high variance advantage estimation can affect joint policy improvement.



Finally, we want to note that the explosion of variance is based on the assumption of independent actors. Therefore, recent works that model agents as recurrent networks of actions whose distributions depend on the previous agents do not suffer from the high variance problem. 
Examples of this approach include Multi-Agent Transformer \citep{wen2022multi} which uses a Transformer-based architecture to model the agents' policies, and \cite{zhao2023local} that explicitly parameterizes the policies to conditionally depend on the actions of prior agents. 
Conditional actors bypass reweighting the old advantage by subsuming the importance ratio into the conditioning of previous actions. 
However, conditional actors require the policies to be sequence models that output each agent's action autoregressively, and do not execute decentrally; the policy needs to process the actions of other agents before executing its own. 


\section{Method}
\label{sec:method}
\subsection{Multi-agent trust region policy optimization with
low variance guarantee}
As we have shown in the previous section, the importance ratio can be the potential cause for the large variance of the advantage estimates. In this section, we propose a simple modification to the current Trust Region objective for MARL that guarantees stronger control over the variance of its estimate.
We start by first defining a new surrogate loss function, obtained by employing a clipping trick to the importance ratio of the advantage in (\ref{eq:advest}). 
\begin{definition}\label{def:cliploss}
    Let $\boldsymbol{\pi}$ be a joint policy, $\boldsymbol{\widetilde \pi}^{{1:m-1}}$ be another joint policy of the first $m-1$ agents, and $\hat{\pi}^{m}$ be some other policy of agent $m$. Define
    \begin{flalign*}\label{eq:surrloss} 
    \begin{multlined}
        L^{{1:m}, \mathrm{clip}}_{\boldsymbol{\pi}}(\boldsymbol{\widetilde{\pi}}^{{1:m-1}}, \hat{\pi}^{m}) 
        \triangleq
        \mathbb E_{s\sim d_{\rho}^{\boldsymbol{\pi}}, \boldsymbol{a}\sim\boldsymbol{\pi}} 
    \bigg[
    \Big(
    \frac{\hat{\pi}^{m}(a^{m}|s)}{\pi^{m}(a^{m}|s)}-1
    \Big) 
    {
    \min\big(}\Re^{{1:m-1}}_{\boldsymbol{\widetilde{\pi}}, \boldsymbol{\pi}}(s, \boldsymbol{a}) 
    {,1
    \big)}
    A_{\boldsymbol{\pi}}(s, \boldsymbol{a})\bigg].
    \end{multlined}
    \end{flalign*}
\end{definition}
Our trust region algorithm using the newly introduced objective is described in Algorithm \ref{alg:main}. {Two specific details that differ our method with the standard trust region presented in (\ref{eq:harl_obj}) are: (i) the trust region term is penalized in expectation of the discounted marginal state distribution, and (ii) the KL divergence is in reverse order.
The reverse order is necessary for the convergence discussed in Theorem \ref{theorem:convergence}.}

\begin{algorithm}[tb]
\caption{Multi-agent trust region policy optimization with low variance guarantee}\label{alg:main}


\begin{algorithmic}
\State {\bfseries Input:} Initial uniformly random joint policy $\boldsymbol{\pi}_0=(\pi_0^1, \dots , \pi^n_0)$, number of iteration $K$
\State {\bfseries Output:} Uniformly sample $k$ from $0, 1, \dots, K-1$, return $\boldsymbol{\bar \pi} = \boldsymbol{\pi}_k$
\For{$k = 0, 1, \dots$, K-1}
\State Compute the advantage function $A_{\boldsymbol{\pi}_k}(s, \mathbf a)$
\State Choose step size $\beta_k$
\State Randomly permute all agents
\For{$m=1, \dots n$}
    \State $\pi_{k+1}^{m} \leftarrow \argmax_{\pi^{m}} \big[L^{{1:m}, \text{clip}}_{\boldsymbol{\pi}_k}(\boldsymbol{{\pi}}_{k+1}^{{1:m-1}}, {\pi}^{m})-
    \beta_k
    \mathbb E_{\rho_{\boldsymbol{\pi}_k}}{D}_{\text{KL}}( \pi^{m}\|\pi_k^{m})\big]$
\EndFor
\EndFor
\end{algorithmic}
\end{algorithm}


{Since ratio sampling is the main source of high variance, clipping this term results in a more controllable instability.}
It is easy to see that the new estimate offers a much more stable variance, in that it does not grow as the number of agents increases. For any state $s \in \mathcal S$, the variance of the clipping advantage is bounded by
\begin{flalign*}
\mathrm{\textbf{Var}}_{\boldsymbol{a} \sim\boldsymbol{\pi}}
    \min\big(\Re^{{1:m-1}}_{\boldsymbol{\widetilde{\pi}}, \boldsymbol{\pi}} 
    ,1
    \big)
    A_{\boldsymbol{\pi}}(s, \boldsymbol{a}) 
     &\leq  
     \mathbb E_{\boldsymbol{a} \sim\boldsymbol{\pi}}
     \min\big(\Re^{{1:m-1}}_{\boldsymbol{\widetilde{\pi}}, \boldsymbol{\pi}} 
    ,1
    \big)^2
    A_{\boldsymbol{\pi}}(s, \boldsymbol{a})^2 \\
    &\leq \mathbb E_{\boldsymbol{a} \sim\boldsymbol{\pi}} \Big[ 
        A_{\boldsymbol{\pi}}
    \Big]^2 
    \leq \epsilon^2.
    \end{flalign*} 
{The ratio clipping serves to relax the policy updates in later iterations, particularly in situations where there is the potential for a drastic increase in estimate noise.}
{As a result, a direct consequence of the new objective is that it does not suffer from the potentially high variance that plagues the original loss as discussed in \ref{prop:exponentialexample}. }

\begin{table*}[ht]
    \centering    
    \caption{Comparison of different trust region MARL algorithms.}
    \label{tab:objective}
    \begin{tabular}{ 
  |c|c|}
     \hline
     \textbf{Method} & \textbf{Objective} \\
     \hline
     MAPPO \cite{yu2022surprising} &
     $\mathbb E 
     \big[ \min\big(
     r^m A_{\boldsymbol{\pi}},
     g(r^m, \epsilon_1) A_{\boldsymbol{\pi}}
     \big)
     \big]$ 
     \\
     \hline 
     HAPPO \cite{kuba2021trust} &
     $\mathbb E
     \big [
     r^m \Re^{{1:m-1}}_{\boldsymbol{\widetilde{\pi}}, \boldsymbol{\pi}} A_{\boldsymbol{\pi}}
     \big] - C\text{D}_{\text{KL}}^{\max}(\pi^{m}, \widetilde{\pi}^{m})$
      \\
     \hline
     CoPPO \cite{wu2021coordinated} &
     $\mathbb E
     \min \big [ 
     r^m g(\Re^{{-m}}_{\boldsymbol{\widetilde{\pi}}, \boldsymbol{\pi}}, \epsilon_2) A_{\boldsymbol{\pi}}, 
     g(r^m g(\Re^{{-m}}_{\boldsymbol{\widetilde{\pi}}, \boldsymbol{\pi}}, \epsilon_2), \epsilon_1) A_{\boldsymbol{\pi}}
     \big]$
      \\
     \hline
     This work & 
     $\mathbb E 
     \big[ r^m
     \min(\Re^{{1:m-1}}_{\boldsymbol{\widetilde{\pi}}, \boldsymbol{\pi}}, 1)
     A_{\boldsymbol{\pi}}
     -\beta_k D_{\text{KL}}(\widetilde{\pi}^m\|\pi^m)
     \big]
     $ \\
    \hline
    \end{tabular}
\end{table*}

\subsection{Convergence Analysis}
Furthermore, similar to the monotonic improvement bound given in (\ref{eq:harlbound}), we justify the use of the new surrogate objective function for the multi-agent reinforcement learning problem by the following monotonic improvement bound.
\begin{restatable}[Policy Improvement Bound]{lemma}{improvementbound} \label{lem:improvementboudclip}
For any two joint policy $\boldsymbol{\pi}$ and $\boldsymbol{\widetilde\pi}$, 
and let $\alpha^i=\max_s\|\pi^i - \widetilde{\pi}^i\|_1$, then 
\begin{flalign*}
    \begin{multlined}[t]
    J(\boldsymbol{\widetilde\pi}) 
\geq  J(\boldsymbol{\pi}) 
         + \frac{1}{1-\gamma}\sum_{m=1}^N
         \Big[ 
         L^{1:m, \mathrm{clip}}_{\boldsymbol{\pi}}(\boldsymbol{\widetilde{\pi}}^{1:m-1}, \widetilde{\pi}^m) 
    -  C
    \alpha^m\sum_{i=1}^{N}\alpha^i 
    - \epsilon \alpha^m \sum_{i=1}^{m-1} \alpha^i
    \Big],
    \end{multlined}
\end{flalign*}
where $C=\frac{4\gamma \epsilon} {1-\gamma}$ and $\epsilon=\max_{s, \boldsymbol{a}}|A_{\boldsymbol{\pi}}(s, \boldsymbol{a})|.$
\end{restatable}
The proof of Lemma \ref{lem:improvementboudclip} is straightforwardly extended from the previous results in \cite{schulman2015trust, kuba2021trust} and is given in detail in the Appendix \ref{proof:policyimprovement}.
{This bound shows that the clipping objective lower bounds the performance difference of the two joint policies and becomes more accurate as the policies become closer. 
The last term $\epsilon\sum_{m, i<m} \alpha^m \alpha^i$ can be seen as a trade-off for ensuring low variance compared to the bound in (\ref{eq:harlbound}).}
Finally, given that the objective in (\ref{def:cliploss}) is not an unbiased estimator of the true advantage, it is natural to ask whether Trust Region methods can converge with this new objective. In the next Theorem, we answer this question affirmatively by showing that the Algorithm \ref{alg:main} converges to approximate $\epsilon$-NE at a sublinear rate that matches the rate established in other Trust Region methods
\begin{restatable} {theorem}{convergence} \label{theorem:convergence}
After $K$ iterations, the policy $\boldsymbol{\bar \pi}$ obtained from the Algorithm \ref{alg:main} is an $\varepsilon_K$-Nash equilibrium, with 
\[\varepsilon_K = O\left(
    \frac{\bar \epsilon\sqrt{N\log A }}{(1-\gamma)\sqrt{K}}
    \right),
\]
where we define $\bar \epsilon= \max_{\boldsymbol{\pi}, s, \boldsymbol{a}} |A_{\boldsymbol{\pi}}(s, \boldsymbol{a})|$ and the notation $O(\cdot)$ hides constant factors.

\end{restatable} 
The proof of the Theorem \ref{theorem:convergence} follows the analysis in \cite{zhao2023local}, 
but with the convergence policies as NE instead.
Before presenting the main proof, we need to establish the following two lemmas.
\begin{lemma}[One-step Descent]\label{lemma:onestep}
For two consecutive policies $\pi^m_{k}$ and $\pi^m_{k+1}$, and an arbitrarily policy $\pi^m$, then 
\begin{flalign*}
    \begin{multlined}[t]
    D_{\text{KL}}\left(\pi^m(\cdot | s)\| \pi^m_{k}(\cdot|s)\right) 
    - D_{\text{KL}}\left(\pi^m(\cdot | s)\| \pi^m_{k+1}(\cdot|s)\right) \geq 
        \left\langle
    \beta_{k}^{-1} \mathbb E_{\boldsymbol{\pi}_{k+1}^{1:m-1}} Q^{1:m}_{\boldsymbol{\pi}_k}, \pi^m - \pi^m_{k+1}
    \right\rangle  \\
    - 2(\beta_k^{-1} \bar \epsilon)^2 m
    \end{multlined}
\end{flalign*}
\end{lemma}
The detailed proof is presented in Appendix \ref{proof:convergencerate}.
{In the proof, we first construct a bound on the estimated advantage using the clipped objective and the true advantage through the L1 distance of the joint policy and its updated one.}
{We further bound the above distance as a ratio to the step size $\beta_k$ at $k^{th}$ update and the maximum advantage $\bar \epsilon$.}
The one-step descent lemma is required for the main convergence proof, as it formalizes the discrepancy of two subsequent policies to another arbitrary policy, which can be the optimal one.
Equivalent forms of 
the lemma appears in several works on the convergence of Trust Region methods in the literature \cite{zhao2023local, liu2019neural, huang2021neural}. 
\begin{lemma}[Difference lemma \cite{kakade2002approximately}]  \label{lemma:difference} For any two joint policy $\boldsymbol{\pi}$ and $\boldsymbol{\widetilde{\pi}}$, then
\[J(\boldsymbol{\widetilde\pi})-J(\boldsymbol{\pi})=\frac{1}{1-\gamma}\mathbb E_{s\sim d_{\rho}^{\boldsymbol{\widetilde\pi}}, \boldsymbol{a} \sim \boldsymbol{\widetilde\pi}}A_{\boldsymbol{\pi}}(s, \boldsymbol{a})\] 
\end{lemma}
{The difference lemma is well-established in single-agent settings} \citep{kakade2002approximately, schulman2015trust}.
{It provides a connection between the performance disparity of two arbitrary policies through the expected advantage of one policy under the distribution of the other.
In the context of multi-agent settings, we can apply the lemma directly to the two joint policies as in a single one.}

Based on Lemma \ref{lemma:onestep} and \ref{lemma:difference}, we prove the convergence result of our algorithm. Specifically, we leverage the discrepancy between two consecutive update sequences established in the Lemma \ref{lemma:onestep}{, and sum this discrepancy over the update iterations. 
This summation collapses through a telescoping sum,
which is then transformed into the difference in performance by the difference lemma. 
The full proof of Theorem} \ref{theorem:convergence} {can be found in the Appendix} \ref{appendix:proofs}.

Theorem \ref{theorem:convergence} establishes a sublinear convergence rate of our method. Compared to the prior rates in Trust Region MARL, 
\cite{zhao2023local} demonstrate a non-asymptotic rate of $O\left(\frac{N\sqrt{\log A}}{(1-\gamma)\sqrt{K}}\right)$ of MAPPO to globally optimal policies with sequential update. 
Notably, a critical condition for globally optimal convergence in \cite{zhao2023local} is that the policies are conditionally dependent on the actions of previous agents.
Under this condition, it is reasonable that global optimality can be achieved with sufficiently powerful policy networks.
The global convergence of dependent actors aligns with the convergence results in single-agent unregularized MDP \cite{shani2020adaptive, liu2019neural}.

On the other hand, modeling agents as independent actors, \cite{kuba2021trust} establish the convergence to NE in the limit of HATRPO.
In Markov Potential Games, \cite{leonardos2021global} show the convergence rate of $O(1/\varepsilon^2)$ to $\varepsilon$-Nash of Policy Gradient methods, and \cite{ding2022independent} improve on the distribution mismatch of the rate.
Regarding the optimality of the Equilibrium,  \cite{papoudakis2020benchmarking} and \cite{christianos2022pareto} suggest that Multi-agent Policy Gradient methods can converge to sub-optimal equilibria, which is also confirmed in \cite{zhao2023local}. 
{While }\cite{zhao2023local}{ demonstrate the optimal convergence with PG, it is required that the agents have to communicate their actions, which we do not consider in the CTDE setting. }
To our awareness, whether Trust Region methods with independent actors can converge to globally optimal policy is still an open question.

\subsection{Discussion}
From the practical implementation perspective, we give a details comparison of several related Trust Region methods in Table \ref{tab:objective}.  We denote $r^m=\frac{\widetilde{\pi}^{m}(a^m|s)}{\pi^{m}(a^m|s)}$ the ratio of current agent, $\epsilon_1, \epsilon_2$ are clipping parameters, and $g(\cdot, \epsilon)=\text{clip}(\cdot, 1-\epsilon, 1+\epsilon)$ the clipping operation of PPO.  We further define
\begin{equation*}
\Re^{-m}_{\boldsymbol{\widetilde{\pi}}, \boldsymbol{\pi}}=\prod_{i\neq m}\frac{\widetilde{\pi}^i(a^i|s)}{\pi^i(a^i|s)}
\end{equation*}
as the ratio used in Coordinated PPO (CoPPO) \cite{wu2021coordinated}.

\textbf{MAPPO} updates each agent based on the advantage estimate, and the Trust Region clipping operator is applied to the ratio of the policy of the agent itself. The main advantage of MAPPO is its sharing parameters mechanism. At each iteration, the samples from all agents are aggregated into a single large batch to reduce variance. The Trust Region in MAPPO is considered independently across agents, where one agent is not aware of the changes the others are making. 

\textbf{HAPPO} incorporates the additional ratio term $\Re^{1:m-1}_{\boldsymbol{\widetilde{\pi}}, \boldsymbol{\pi}}$ from other agents into the optimization objective of individual agents. The sequential update framework of HAPPO is general, with HAPPO as one concrete implementation of the Trust Region mechanism. In sequential updates, the changes of the previous agents are accumulated, and thus in later agent optimization, the actions are weighted based on the accumulated changes; If the prior agents promote their actions, then the current agent is also likely to promote theirs due to the high weights assigned by the cumulative ratio and vice versa, encouraging cooperation between agents. The main benefit of HAPPO is in the heterogeneous domains, where the state and action spaces of agents can differ. 

\textbf{CoPPO} \cite{wu2021coordinated} is designed specifically to coordinate consecutive PPO iterations among multiple agents. CoPPO is built upon MAPPO and thus inherits the parameter-sharing archetype of MAPPO.. The formulation for the coordination of CoPPO appears similar to the ratio term in HAPPO and is formulated as the weight of the cumulative ratio of other agents $\Re^{-m}_{\boldsymbol{\widetilde{\pi}}, \boldsymbol{\pi}}$. The difference is that this ratio considers \textit{all} other agents, this is because agents share their parameters, and updating one agent will affect others as well. To tackle the high variance problem of the coordination, a double clipping trick is used on both the current ratio term (the trust region) and the cumulative term. As a result, CoPPO is inherently coupled with the implementation of MAPPO and the sharing-parameter paradigm.

\begin{figure*}[ht]
\centering\includegraphics[width=1.\columnwidth]{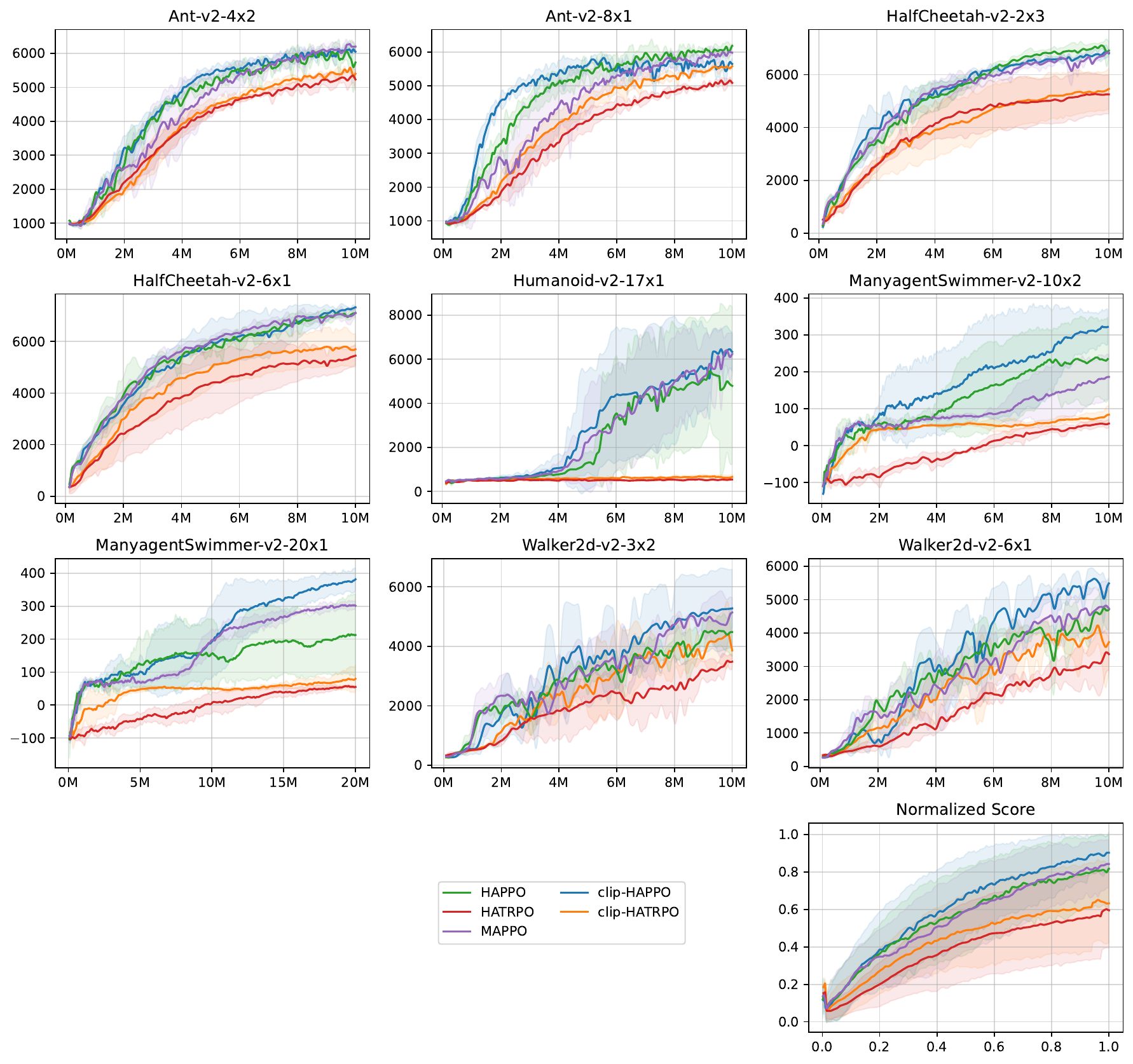}
  \caption{Results on 9 environments of MaMuJoCo benchmark.}
  \label{fig:mamujoco}
\end{figure*}

Our work is targeted at a more general case of Trust Region frameworks where the cumulative ratio quickly becomes noisy for the later agents in sequential updates.
The clipping operator is selected due to its simplicity in controlling the variance with a tradeoff in the joint policy improvement.
Moreover, we establish a sublinear convergence rate to $\epsilon$-NE.
For practical applications, the experiment results demonstrate the superiority of having a low variance advantage estimation.
The combination of our clipping objective and the well-known PPO is a strong baseline for Trust Region-based approaches in MARL, achieving more stable results than its competitors such as MAPPO and HAPPO.

\section{Experiments}
\label{sec:experiments}
In this section, we design our experiments to 
answer the following questions:
(i) Does our clip objective improve the MARL performance in the heterogeneous setting with trust region-based methods? 
(ii) How does reducing high variance improve the performance?
(iii) Do the benefits of our clip objective transfer to other settings, including fixed-order training and parameter sharing?

We evaluate the performance of our method on three widely used benchmarks:
\begin{itemize}
    \item {\textbf{Multi-Agent MuJoCo (MAMuJoCo)}. The MAMuJoco benchmark \cite{de2020deep, peng2021facmac} extends the MuJoCo benchmark to the multi-agent domains by designating the controls of different parts of a robot to different, independent agents. Each environment is designed with a specific moving model and number of agents. The number of agents ranges from two agents to more than 20 agents as in the \textit{ManyagentSwimmer-v2} setting. For benchmarking purposes, we test our methods and the selected baselines on 9 environments of MAMujoCo.}
    \item {\textbf{StarCraft II.} The StarCraft Multi-Agent Challenge (SMAC) \cite{samvelyan2019starcraft} presents a set of scenarios, each consisting of two opposing teams that aim to defeat the other with varying levels of difficulty. The benchmark emphasizes the micromanagement of individual agents and decentralized testing. We consider several hard and super hard tasks for evaluation, with both homogeneous and heterogeneous settings.}
    \item {\textbf{Multi-Agent Particle Environment (MPE)}. The MPE benchmark  \cite{lowe2017multi} is the set of simple scenarios on a 2D plane with particle agents that can interact with each other and stationary landmarks.
    These scenarios vary from fully cooperative to mixed competitive setups.
    We focus on three tasks for their fully cooperative nature: Spread, Refenrence, and Speaker-Listener.}
\end{itemize}

We implement our methods using the open-source codebase of HARL \cite{zhong2023heterogeneous}. Our methods are simple and can be straightforwardly integrated into the existing implementations, with only modifications to the objective functions of the corresponding algorithms, and introduce no additional hyperparameters.
{Source code is available at }
\url{https://github.com/giangbang/Low-Variance-Trust-Region-MARL.}

We compare our methods, clip-HAPPO, and clip-HATRPO, against state-of-the-art Trust Region Policy Gradient MARL algorithms, including HAPPO, HATRPO \cite{kuba2021trust}, and MAPPO \cite{yu2022surprising}. All algorithms are implemented as parameter-independent in all experiments.
For both clip-HAPPO and clip-HATRPO, we adopt the settings from \cite{zhong2023heterogeneous}.
In general, the network architecture is set to 2 layers MLP with the hidden size of 128 in MPE, 3 layers with the hidden size of 128 in MaMujoco, and 64 in SMAC. When RNN is used, a GRU layer is attached after the same MLP backbone. The hyperparameters gamma ($\gamma$) and state type in SMAC (FP and EP) are not tuned and followed \cite{zhong2023heterogeneous}.
All the parameters used in our experiments are reported in the Appendix Section \ref{appendix:experimentdetails}.
The hyperparameters of the compared algorithms are taken from the respective papers and the provided tuned config from the above library.

Our experiments are carried out on a workstation with NVIDIA GeForce RTX 3090 24G and 64G main memory.
Each run takes from 30 minutes to 10 hours, depending on the environment and number of agents.
For each environment and model, we repeat at least three different seeds.


\subsection{Experiment results} 


\textbf{MaMuJoCo}.
As shown in Figure \ref{fig:mamujoco}, among the 9 tested environments, clip-HAPPO achieves the best performance compared to other baselines, especially in environments with a high degree of variability includes \textit{Walker2d} and \textit{ManyAgentSwimmer}. On the other hand, clip-HATRPO consistently outperforms HATRPO in \textit{all} tested environments. 
The performance gain of the TRPO-based approach with clipping objective is significant in most of the environments and even comparable to MAPPO. 
Notably, unlike HAPPO, HATRPO uses KL divergence for optimization, making it more sensitive to extreme policy updates. This sensitivity amplifies their impact on TRPO iterations compared to PPO.
Low variance estimates reduce the extreme updates of the TRPO iterations, improving the convergence rate of the algorithms. 
 

\textbf{SMAC}. We evaluate our methods on three super hard environments of SMAC. Figure \ref{fig:smac} demonstrates that clip-HAPPO outperforms all other methods in 2 out of 3 tested scenarios, and clip-HATRPO demonstrates its effectiveness over HATRPO in all environments. Specifically, in \textit{27m\_vs\_30m}, the instability is observed in HAPPO but is completely avoided in our method. This highlights the stability of our method in environments with a large number of agents. The instability effect is absent in TRPO-based methods, possibly because the trust region is explicitly enforced in the probability space. Therefore, HATRPO is more stable than HAPPO, albeit with the high variance issues, but learns slower. We again observe the advantage of our clip-HATRPO over HATRPO in all scenarios. MAPPO performs slightly better than our method in \textit{27m\_vs\_30m}, likely because it also does not suffer from high variance.
However, MAPPO falls short in non-homogeneous environment \textit{3s5z\_vs\_3s6z}.


\begin{figure*}[htbp]
\centering\includegraphics[width=1.0\columnwidth]{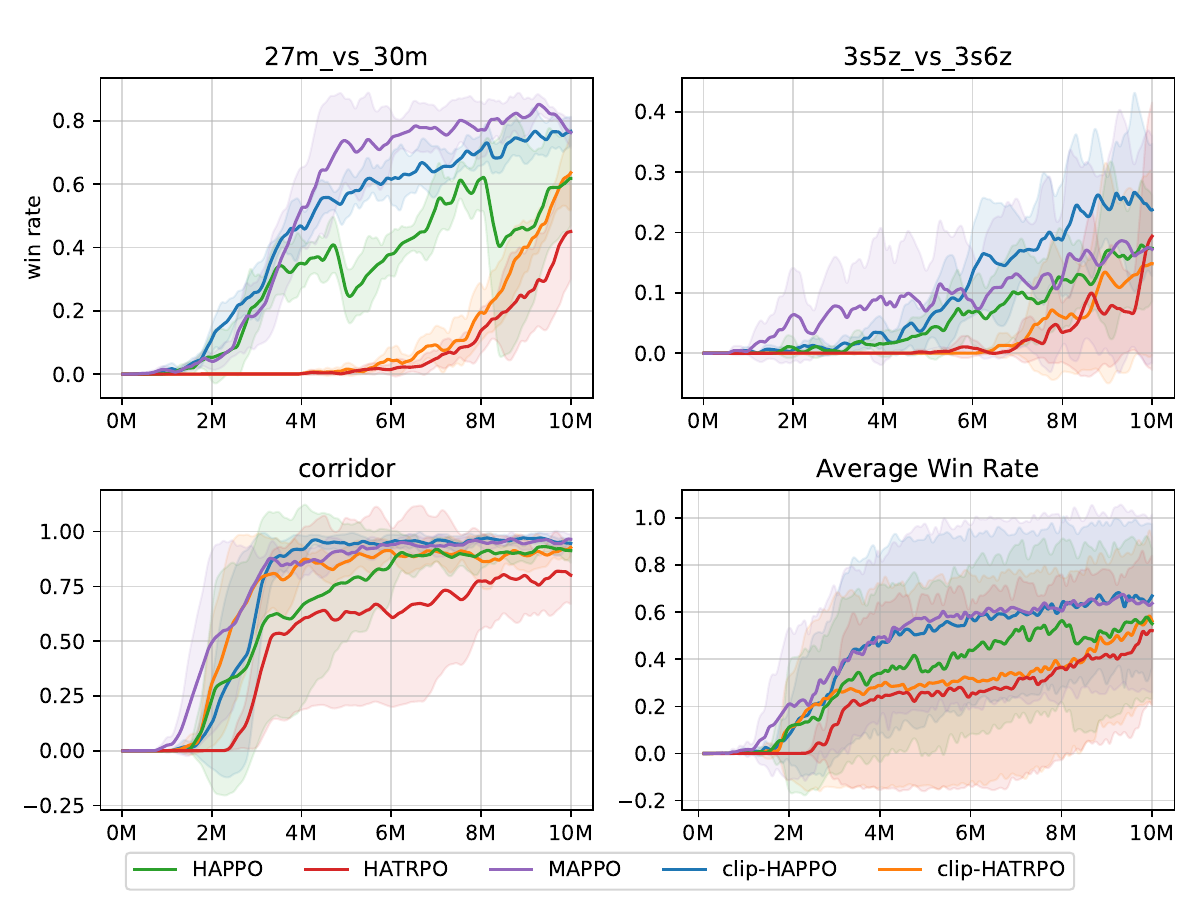}
  \caption{Results on 3 super hard environments of SMAC benchmark.}
  \label{fig:smac}
\end{figure*}
\begin{figure*}[htbp]
\centering\includegraphics[width=1.0\columnwidth]{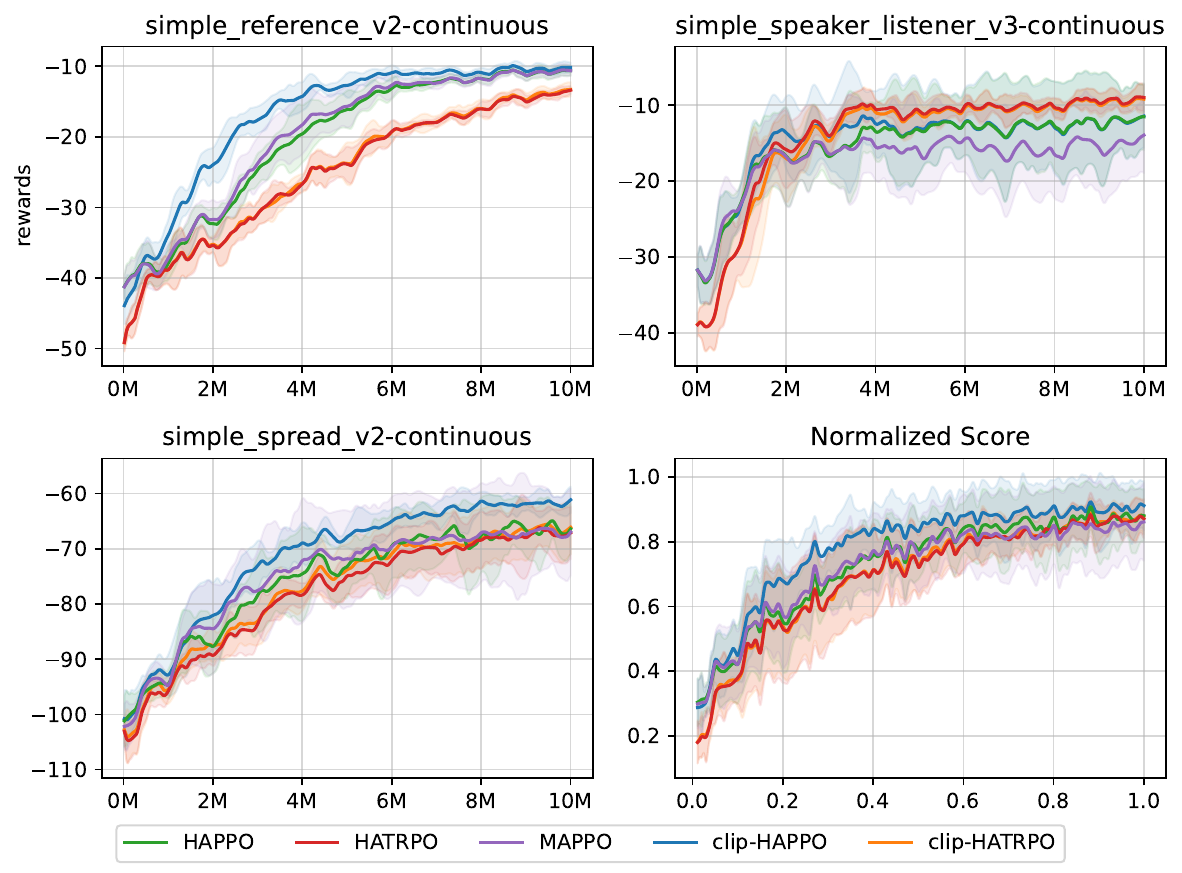}
  \caption{Results on 3 environments of MPE benchmark.}
  \label{fig:mpe}
\end{figure*}

\textbf{MPE}. The three continuous, fully cooperative tasks in MPE are considered. 
From Figure \ref{fig:mpe}, we can see that clip-HAPPO achieves the best performance in 2 of 3 tested environments while the HATRPO and clip-HATRPO have the best performance in the remaining \textit{Speaker-Listener} environment. Note that, since the tasks in MPE are relatively simple with at most three agents, the unstable advantage estimation does not affect the learning process severely. Therefore, the variations between tested methods are negligible.

\subsection{Ablation study}

\begin{figure}[htbp]
\label{fig:mamujoco_variance}
\centering\includegraphics[width=.7
\columnwidth]{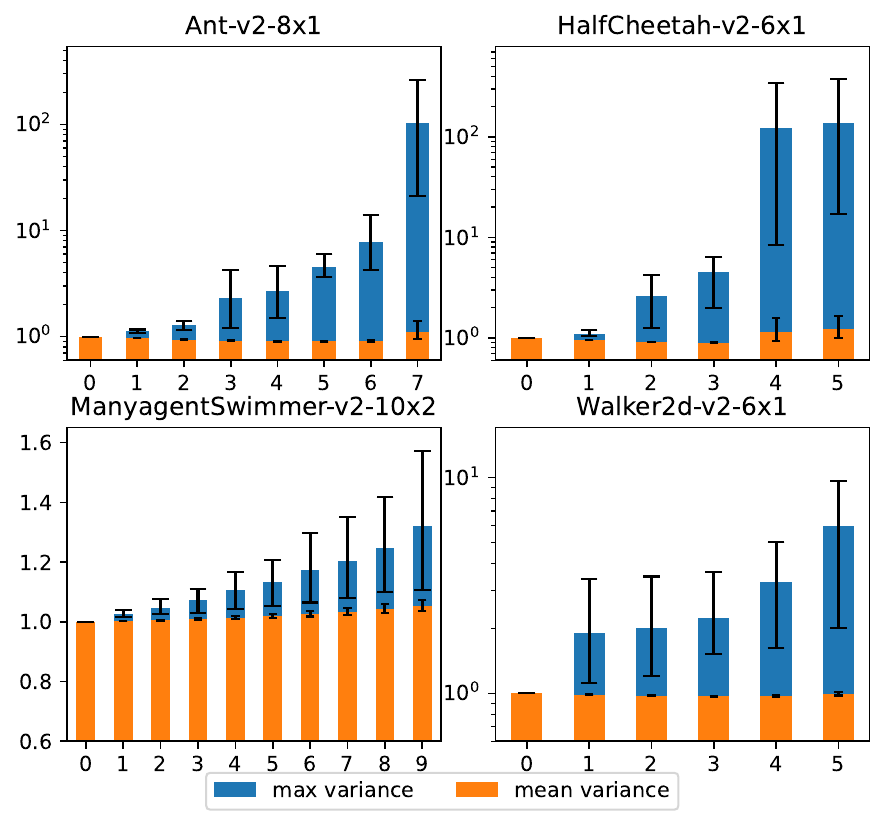}
  \caption{Variance of HAPPO on different MaMuJoCo tasks, with error bars maximum and minimum over seeds. \label{fig:mamumjocovariance}}
\end{figure}

\begin{figure}[htbp]
\label{fig:mamujoco_variance_clip}
\centering\includegraphics[width=.7
\columnwidth]{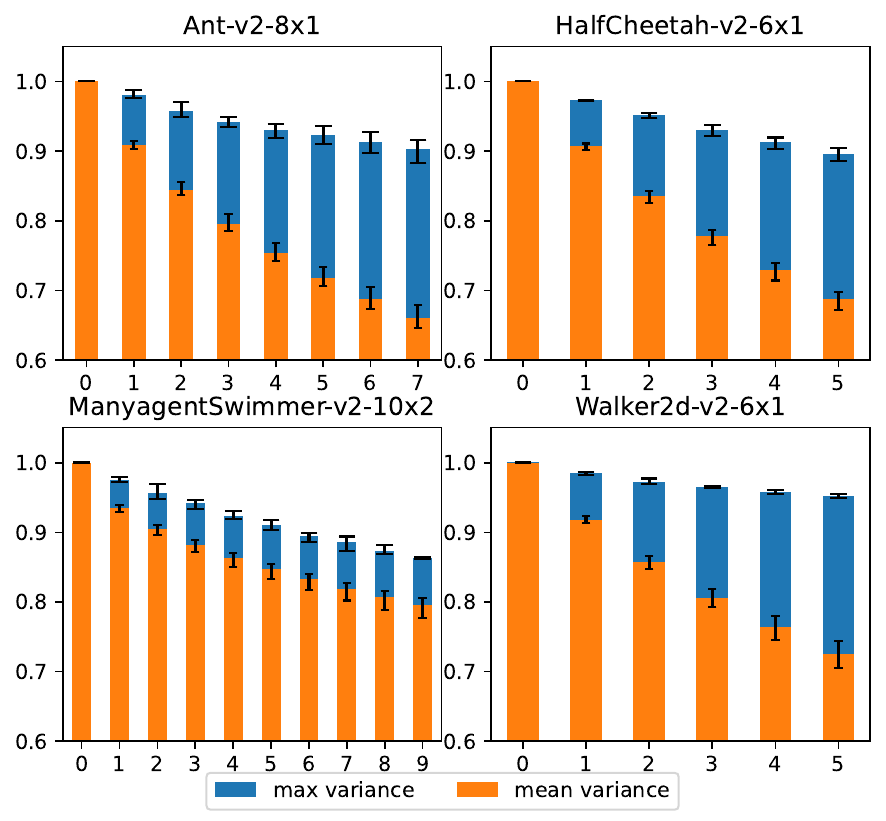}
  \caption{Variance of clip-HAPPO on different MaMuJoCo tasks, with error bars maximum and minimum over seeds. \label{fig:mamumjocovarianceclip}}
\end{figure}

\begin{figure*}[htbp]
\label{fig:mamujoco_thread}
\centering\includegraphics[width=
1\columnwidth]{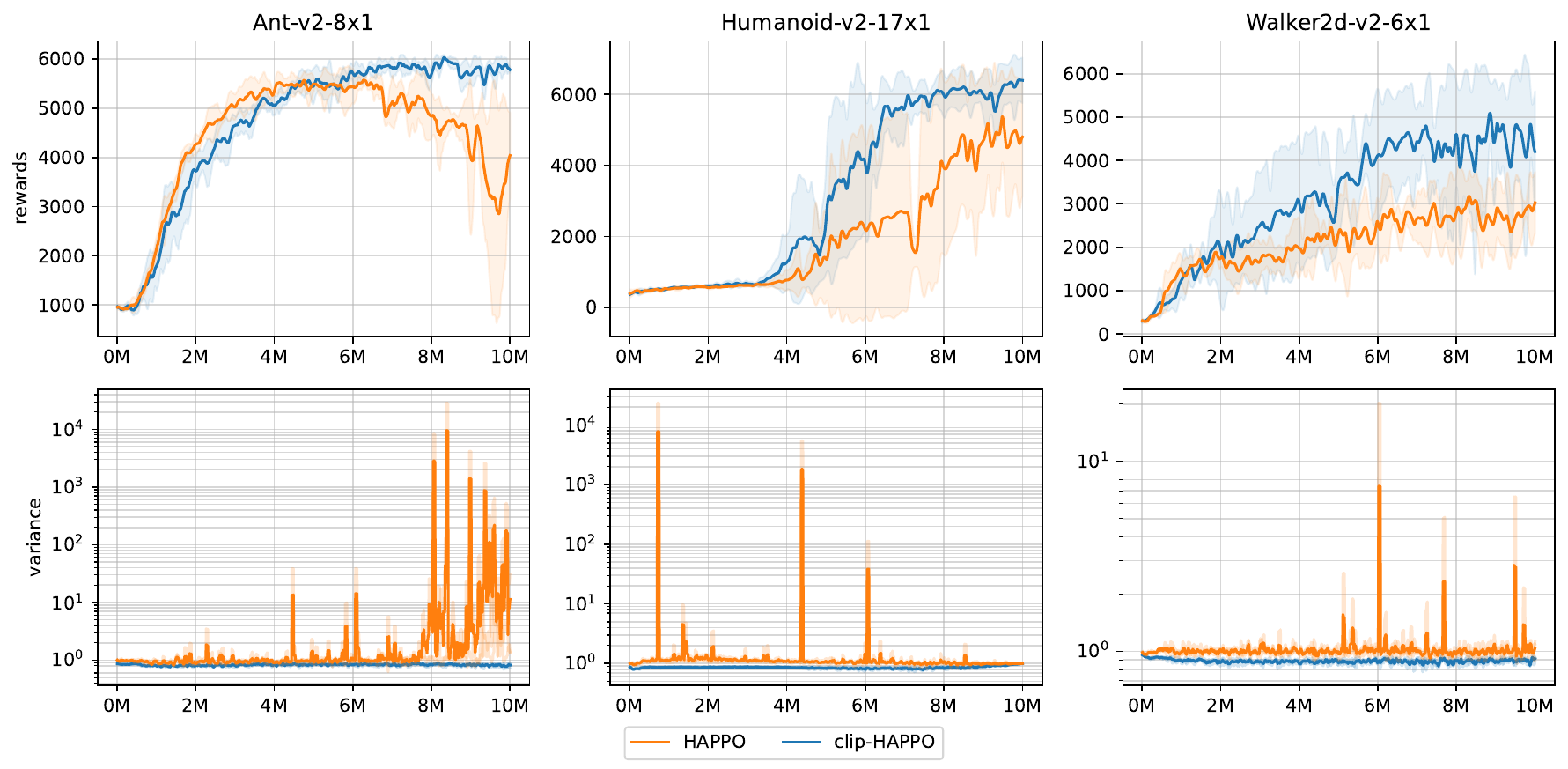}
 \caption{Reward and advantage variance on MaMuJoCo with fewer threads. The variance plots are in log-scale} \label{fig:10threadplots}
\end{figure*}

{
Due to the superiority of the PPO-based methods, including HAPPO and our version clip-HAPPO, we only perform ablation experiments with the two methods.}

\begin{figure}
\label{fig:mamujoco_fix_order}
\centering\includegraphics[width=.7\columnwidth]{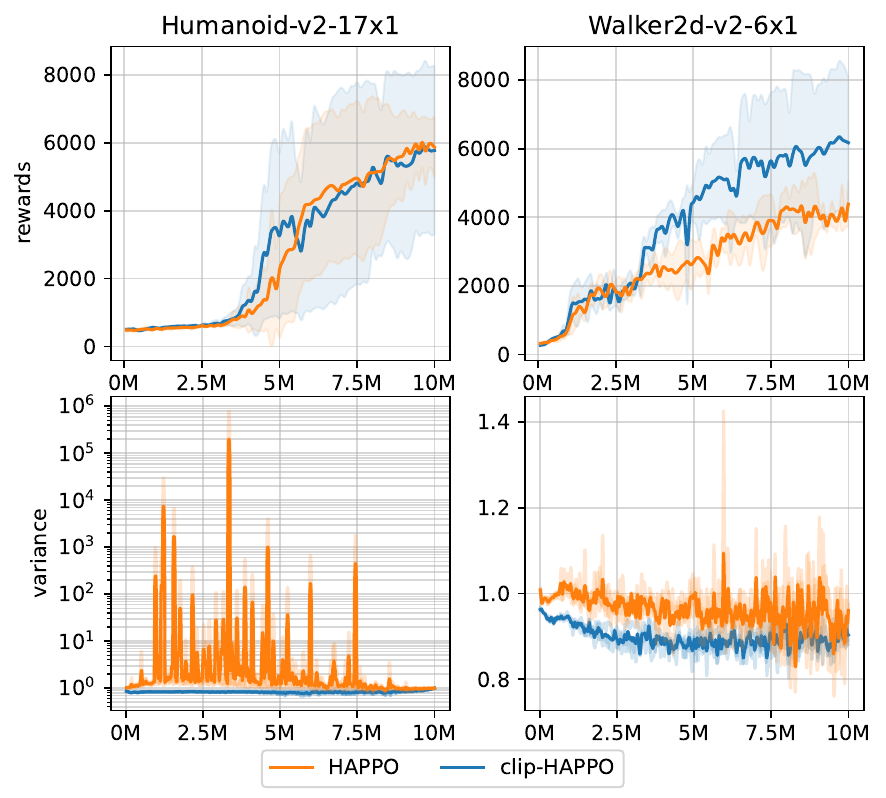}
  \caption{Reward and advantage variance on MaMuJoCo with fixing agents' order. The variance is reported as the average value of the agents.
  \label{fig::appendixfixorder}}
\end{figure}

\textbf{Effect of the clipping objective on the variance.} We investigate the impact of high-variance updates on performance across diverse environments in the MaMuJoCo benchmark for HAPPO and compare it with clip-HAPPO. We observe that the exponential variance problem arises in all settings, while in our method, the variance tends to decrease with later agents. However, we also see that not all high-variance updates are detrimental.
Figure \ref{fig:mamumjocovariance} and Figure \ref{fig:mamumjocovarianceclip} give illustrations of the agent's advantage variance of the baseline HAPPO and our clip-HAPPO, respectively.
{For each method, the advantage  $A_\pi(s, a)$ is estimated through data samples and normalized to have 1 standard deviation at the first agent.
For each $m^{th}$ agent in the update order, the corresponding advantage $A^m_\pi(s, a)$ is computed according to the corresponding training objective.
}
It can be seen clearly the exponential growth trend of the HAPPO baseline.
The maximum variances are varied across the environments, ranging from $10^1$ in \textit{Walker2d-v2-6x1} to $10^2$ in \textit{Ant-v2-8x1} and \textit{HalfCheetah-v2-6x1}.
The growth rate depends on the specific settings.
For example, the maximum variances in \textit{Walker2d-v2-6x1} are more stable than those of \textit{HalfCheetah-v2-6x1}.
In the case of \textit{Humanoid-v2-17x1}, the maximum variances can reach more than $10^6$ as can be seen in Figure \ref{fig:exponential_figure}.
Meanwhile, our proposed method has a more stable advantage variance of around 1 across different environments.

{Furthermore, to highlight the superiority of clip-HAPPO's low variance property, we reduce the number of training rollout hyperparameters for each update.}
{The number of training rollout threads can have an impact on the variance of the algorithms, since with more samples from
the environment, the estimation becomes more accurate, therefore
reducing the effect of high variance.}
{To magnify this effect, we intentionally lower the number of parallel environments used for simulation by using a lower number of training threads. In this way,}
the effects of low-variance updates can be seen more clearly.
In Figure \ref{fig:10threadplots}, we plot the training results with 10 threads of our clip-HAPPO and the original HAPPO.
The differences between the two methods are highlighted in this case.
When the issue of variance becomes more severe, our clipped version still convergences as expected.
{We provide further experiments in Appendix Section }\ref{sec:c}. 
{The results highlight that our clip version outperforms HAPPO in 5/6 tested environments with half the number of threads and 4/6 when the threads are further reduced to a quarter.}
In practice, it is worth noting that we generally prefer increasing the number of parallel environments for faster simulation with the trade-off of reducing sample efficiency.


\textbf{Effect of update order}. {We give a comparison between clip-HAPPO and HAPPO with fixing order update in Figure 
\ref{fig::appendixfixorder}.
The update order of all agents are $1,2,\dots,N$ where $N$ is the number of agents.}
However, unlike the report from \cite{kuba2021trust}, we observe few differences between the fixing and random order of updates in the tested environments.
Similar to the random order settings, our method is comparable in \textit{Humanoid-v2-17x1} and better than HAPPO in \textit{Walker2d-v2-6x1}.
HAPPO still suffers the high advantage variance in the later agents and these variance values dominate the plot.
It is worth noting that in the fixed order setting, the advantage of later agents is under the effects of high variance.
Thus, it would result in more unstable policy updates than the random order setting.
On the other hand, our clip-HAPPO has more stable advantage and is observed with little differences between the updating order of agents in the two tested environments.



\textbf{Sharing vs non-sharing parameters. } For completeness of the policy-based approaches, we consider the setting of homogeneous environments where different agents can fully share their parameters.
We compare our clip-HAPPO with MAPPO, CoPPO \cite{wu2021coordinated}, and HAPPO with the parameter sharing settings in the hard and super hard environments of the SMAC benchmark. 
{When all agents' policies share the same network,} the effect of high-variance updates can be exacerbated by the fact that catastrophic updates from one agent can have a negative impact on the others. 
This effect is local in independent-parameters settings, as noisy updates are confined to the individual agent itself.
Therefore, using low-variance updates facilitates safer parameter sharing between agents.
Our clip-HAPPO has the most stable performance across all the tested environments. 
Note that HAPPO is mainly designed to use in the non-sharing setting while MAPPO and CoPPO are designed for the sharing setting.
We implement the sharing procedure such that the agent update cycles are interleaved with the PPO update epochs to avoid the "forgetting" problem of sequential learning with neural networks \cite{kirkpatrick2017overcoming}.
We evaluate all four methods in the three selected environments of SMAC including \textit{8m\_vs\_9m} (Hard), \textit{MMM2} (Super Hard), and \textit{6h\_{}vs\_{}8z} (Super Hard).

Experiment results with parameter sharing demonstrate that sharing parameters with low variance updates mitigates the bad updates of individual agents to the whole.
{In Figure \ref{fig:smac_share}, we see a significant instability of HAPPO in the sharing settings; The performance of HAPPO has the highest variance in all the tested environments, especially in \textit{MMM2} where the algorithm does not converge in some runs. With CoPPO, we observe that the recommended configuration cannot converge in the super hard \textit{6h\_vs\_8z} scenario. Meanwhile, our method achieves competitive and stable performance in all three environments with other baselines, demonstrating the effectiveness of sequential update with the clipping objective in the homogeneous domains.}
{With the competitive results of our methods to other state-of-the-art methods in sharing parameters, and given that this requires homogeneous setups and is not the main focus of our paper, we do not further pursue the performance of the clipping objective in the sharing settings.}

\begin{figure*}
\centering\includegraphics[width=
1.0\columnwidth]{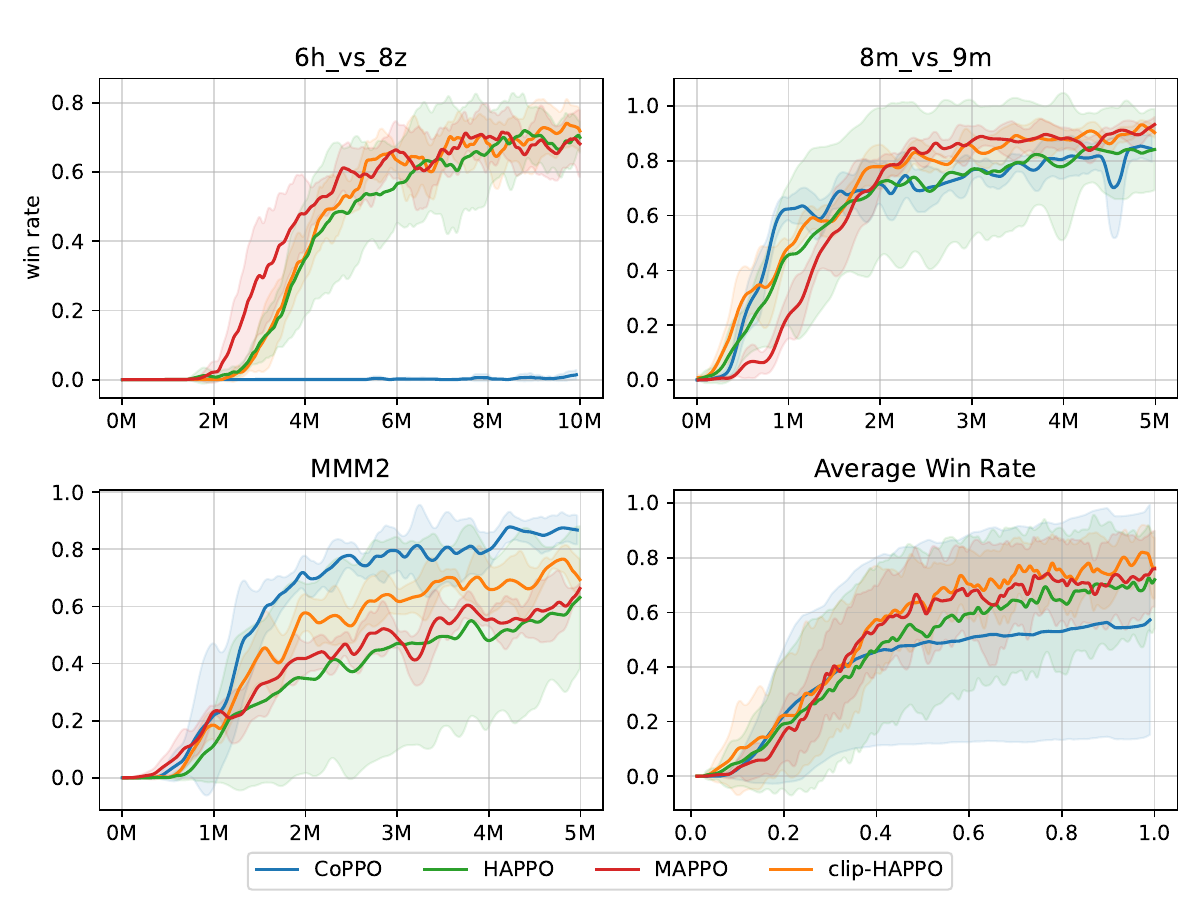}
  \caption{The performance on 3 scenarios of SMAC with sharing parameter settings.}
  \label{fig:smac_share}
\end{figure*}

\section{Related Work}
\label{sec:relatedwork}
\textbf{Multi-agent Trust Region optimization.}
Many on-policy MARL works attempt to emulate the success of Trust region methods from the single-agent domains to the multi-agent settings. Early works, such as IPPO \citep{de2020independent}, which directly employs PPO \citep{schulman2017proximal} independently to each agent and treats other ones as part of the environments, or MAPPO \citep{yu2022surprising} that consider the joint critic for centralized training with parameters sharing for the homogeneous agents. 
Similarly, \cite{kuba2021trust} propose the Trust region methods with the sequential update scheme for multi-agent learning that guarantees the monotonic improvement property, which is our most related work. Following this, \cite{zhong2023heterogeneous} propose a framework on mirror learning for heterogeneous multi-agent learning and extend the sequential updates to a broader class of algorithms. Recently, the multi-agent transformer is introduced by \cite{wen2022multi} based on the close connection between the sequential updates and the sequence modeling problem. 
Other works of independent interest in trust region MARL include the work of \cite{li2023multiagent} that consider a trust region method in distributed settings based on consensus optimization, \cite{li2021dealing} who view trust region decomposition as a way to tackle the non-stationarity problem, and \cite{gu2021multi} where the safety constraint is incorporated into the trust region optimization.


\textbf{Variance reduction in RL.}
Variance reduction has long been an effective way to improve the stability of RL algorithms, most notably the use of a baseline in policy gradient methods \cite{tucker2018mirage, chung2021beyond}. 
In value-based approaches, Retrace($\lambda$) \cite{munos2016safe} is a clipping-based Importance Sampling method on off-policy data to reduce variance.
Moving to multi-agent settings, COMA \cite{foerster2018counterfactual} uses a counterfactual baseline to resolve the credit assignment problem, and \cite{kuba2021settling} derived the optimal baseline for general MARL with a centralized critic. 
In Trust Region MARL, the problem of high variance is suggested by \cite{wu2021coordinated} when multiple PPO iterations coordinate their updates, which the authors address by proposing a double clipping trick on the objective of PPO.
More recently, \cite{wang2023order} propose to use an off-policy correction scheme for sequential update
inspired by Retrace for the monotonic improvements of each agent.
 
{\textbf{Independent vs dependent actors.}
Independent policy algorithms are dominant in multi-agent reinforcement learning for their practical advantages. 
Centralized training with decentralized execution (CTDE) \cite{foerster2018counterfactual} is a paradigm that leverages independent actors with local information at inference time. 
In value-based methods, MADDPG \cite{lowe2017multi} learns a centralized critic and subsequently trains decentralized actors. In particular, monotonic value decomposition methods for learning local value functions \cite{sunehag2017value, rashid2020monotonic} gain popularity in discrete settings.
In Policy Gradient methods, \cite{de2020independent, yu2022surprising, kuba2021trust} are methods that formulate learning agents as independent learners, with possibly a centralized baseline for reducing variance. 
On the theoretical side, 
multiple works study the convergence of policy gradient methods with independent actors, in Competitive Zero-sum Games \cite{daskalakis2020independent}, Markov Potential Games \cite{leonardos2021global} and both \cite{ding2022independent}. 
On the other hand, recent methods for dependent actors include MAT \cite{wen2022multi} which models each agent as an autoregressive model, and \cite{zhao2023local} which parameterizes policies with actions-dependent distributions.
}

\textbf{Convergence rate of RL algorithms.}
In the single-agent setting, \cite{schulman2015trust} show the monotonic improvement of TRPO. Later, \cite{neu2017unified} establish the global convergence of TRPO to the optimal policy, while \cite{liu2019neural} shows the $O(1/\sqrt{K})$ convergence rate of both PPO and TRPO with over-parametrized neural networks. \cite{shani2020adaptive} show a similar nonasymptotic convergence rate of TRPO and an improved $O(1/K)$ rate in regularized MDP. \cite{lan2023policy} proposes a policy mirror descent method that can achieve a linear convergence rate in strongly regularized MDP. Moving on to the multi-agent settings, \cite{kuba2021trust} prove the monotonic improvement property of HATRPO, and building on the monotonicity, they show the convergence to Nash Equilibria in the limit of the algorithm. 
{In Markov Potential Games - a broader class of cooperative Markov Games, 
the $O(1/\sqrt{K})$ convergence rate to NE of Policy Gradient methods with independent actors is established in
\cite{leonardos2021global} and an improved rate by \cite{ding2022independent}. }
More recently, \cite{zhao2023local} show the $O(1/\sqrt{K})$ rate to the joint optimal policy with the sequential update and dependent-action policies. We show that our method can converge to an $O(1/\sqrt{K})$-Nash with sequential updates independent actors while maintaining a low variance guarantee.


\section{Conclusion}
\label{sec:conclusion}
In this paper, we examine the critical problem of high variance in Trust Region methods in MARL with independent actors and sequential update schemes. 
To address this problem, we propose a novel surrogate objective function through a clipping trick that guarantees a low variance estimate. We justify the monotonic bound of the new objective and demonstrate that the Trust region methods can converge to approximate Nash equilibria using the proposed objective. This enables us to derive two new algorithms clip-HAPPO and clip-HATRPO. Experiments on popular benchmarks verify that our methods achieve competitive performance to other strong baselines.

{For the limitation, the exponential growth derived in this work is only the upper bound of the variance, meaning that it is a pessimistic point of view. In practice, this rate of growth can depend on various factors such as environment types or hyperparameters and does not always follow the worst-case scenario. Secondly, though our proposed methods are theoretically backed up, it is not the only way to reduce variance. The training and trust region setups also play an important role in mitigating the variance, 
{which is why HATRPO is less susceptible to the high variance problem than HAPPO.}
Thirdly, even when the high variance updates happen, they also do not always lead to training instability and sometimes can be exploited for exploration.
{As our clipped objective could potentially dampen the exploration signal within the joint advantage, a natural extension of our proposed method would be to vary the clipping threshold to encourage more exploratory behavior in later agents.
In a practical setting, estimating the exploration effect of advantage estimation is challenging due to its entangled interaction with the policy's loss landscape and the policy's entropy
}
Finally, fixing the clipping threshold to 1 for all agents may appear pessimistic in environments with less severe variance issues.

{For future works, based on the results of this paper, it is possible to expand to estimate the convergence rate with sample-based analysis, which would highlight the advantage of our method in terms of low variance estimation.
{Moreover, a promising direction is to study whether the joint policy optimized by the Trust Region framework can converge to the global optimal point under the independent actors setting.}

\subsubsection*{Acknowledgements} This material is based upon work supported by the Air Force Office of Scientific Research under award number FA2386-24-1-4012.

\subsubsection*{Author contributions}
Conceptualization: Bang Giang Le, Viet Cuong Ta; Methodology: Bang Giang Le, Viet Cuong Ta; Formal analysis and investigation: Bang Giang Le, Viet Cuong Ta; Experiments: Bang Giang Le; Writing - original draft preparation: Bang Giang Le; Writing - review and editing: Viet Cuong Ta; Funding acquisition: Viet Cuong Ta; Supervision: Viet Cuong Ta.

\section*{Declarations}

\subsubsection*{Confict of interest}
The authors declare no competing interests.

\begin{appendices}

\section{Omitted Proofs}
\label{appendix:proofs}

We give the details of the omitted results of Theorem \ref{theorem:convergence} and Lemma \ref{lem:improvementboudclip} in the main text in Section \ref{proof:convergencerate} and Section \ref{proof:policyimprovement}, respectively. 
{Section }\ref{sec:varianceproof} {present the proof for the explosion of sampling ratio in the example }\ref{prop:exponentialexample}.

\subsection{Omitted results of Convergence Rate} \label{proof:convergencerate}


\begin{proposition}[Triangular inequality of joint policy]\label{prop:triangle}
    Given two joint policy $\boldsymbol{\pi}^{1:m}$ and $\boldsymbol{\widetilde \pi}^{1:m}$, then the following holds
    \[\|\boldsymbol{\pi}^{1:m} - \boldsymbol{\widetilde \pi}^{1:m}\|_1 \leq \sum_{i=1}^{m}\|\pi^i - \widetilde {\pi}^i\|_1\]
\end{proposition}
\begin{proof}
    Let $\boldsymbol{\bar \pi}^{k, m} = (\boldsymbol{\widetilde {\pi}}^{1:k}, \boldsymbol{\pi}^{k+1:m})$ be a new joint policy that is a combination of the first $k$ agents of $\boldsymbol{\widetilde {\pi}}$ and the rest are taken from $\boldsymbol{\pi}$, then we have the following triangular inequality
    \[\|\boldsymbol{\pi}^{1:m} - \boldsymbol{\widetilde \pi}^{1:m}\|_1 \leq \sum_{i=1}^m \| \boldsymbol{\bar \pi}^{i, m} - \boldsymbol{\bar \pi}^{i-1, m}\|_1\]
    Note that each of the pair $\boldsymbol{\bar \pi}^{i, m}$ and $\boldsymbol{\bar \pi}^{i-1, m}$ only differs by the $i^{th}$ agent, so {$\| \boldsymbol{\bar \pi}^{i, m} - \boldsymbol{\bar \pi}^{i-1, m}\|_1 = \|\pi^i - \widetilde \pi^i\|_1$}, which concludes the proof.
\end{proof}
The above proposition is true regardless of whether the agents act independently or not. 

Now, in addition to the "single-agent" advantage defined in the definition \ref{def:partialQA}, it is useful to introduce the "partial" advantage function similar to the partial Q function, as follows
\[A_{\boldsymbol{\pi}}^{1:m}(s, \boldsymbol{a}^{1:m})\triangleq \mathbb E_{\boldsymbol{a}^{-1:m}\sim \boldsymbol{\pi}}A_{\boldsymbol{\pi}}(\boldsymbol{a}^{1:m}, \boldsymbol{a}^{-1:m}) = Q_{\boldsymbol{\pi}}^{1:m}(s, \boldsymbol{a}^{1:m}) - V_{\boldsymbol{\pi}}(s),\]
it is easy to see that $\max|A^{1:m}_{\boldsymbol{\pi}}(s, \boldsymbol{a}^{1:m})| \leq \max_{s, \boldsymbol{a}, \boldsymbol{\pi}} |A_{\boldsymbol{\pi}}(s, \boldsymbol{a})| = \bar \epsilon$.
Note in the following lemma that $A^{1:m}_{\boldsymbol{\pi}}(s, \boldsymbol{a}^{1:m-1}, a^m)$ should not be confused with $A^{m}_{\boldsymbol{\pi}}(s, \boldsymbol{a}^{1:m-1}, a^m)$.
\begin{lemma} \label{lemma:advantage_dif}
    For two joint policies $\boldsymbol{\pi}$ and $\boldsymbol{\widetilde{\pi}}$, any agent $m\in \mathcal N, s\in \mathcal S, a^m \in \mathcal A^m$, then 
    \begin{flalign*}
    \begin{multlined}[t]
        \left| \mathbb E_{\boldsymbol{a}^{1:m-1}\sim \boldsymbol{{\pi}}} A^{1:m}_{\boldsymbol{\pi}}(s,\boldsymbol{a}^{1:m-1}, a^m) 
        - 
        \mathbb E_{\boldsymbol{a}^{1:m-1}\sim \boldsymbol{\pi}}\min(\Re^{1:m-1}, 1)A^{1:m}_{\boldsymbol{\pi}}(s, \boldsymbol{a}^{1:m-1}, a^m)\right|\\
        \leq \bar \epsilon\sum_{i=1}^{m-1} \|\pi^i - \widetilde {\pi}^i\|_1
        \end{multlined}
    \end{flalign*}
    where $\bar \epsilon = \max |A_{\boldsymbol{\pi}}(s, \boldsymbol{a})|$
\end{lemma}
\begin{proof} In this proof, we omit $(s, \boldsymbol{a}^{1:m-1}, a^m)$ because it is clear from the context
    \begin{flalign*}
        &\left| \mathbb E_{ \boldsymbol{{\pi}}^{1:m-1}} A^{1:m}_{\boldsymbol{\pi}} 
        - 
        \mathbb E_{ \boldsymbol{\pi}^{1:m-1}}\min(\Re^{1:m-1}, 1)A^{1:m}_{\boldsymbol{\pi}}\right| \\
        \leq & 
        \langle  |A^{1:m}_{\boldsymbol{\pi}}|, |\boldsymbol{\widetilde{\pi}}^{1:m-1} - \boldsymbol{{\pi}}^{1:m-1} | \rangle
        \\
        \leq& \|A^{1:m}_{\boldsymbol{\pi}}\|_\infty \|\boldsymbol{\widetilde{\pi}}^{1:m-1} - \boldsymbol{{\pi}}^{1:m-1}\|_1 \\
        \leq& \bar \epsilon\sum_{i=1}^{m-1}\|\widetilde{\pi}^i -\pi^i\|_1
    \end{flalign*}
    where the first inequality is due to the definition of $\Re^{1:m-1}=\frac{\boldsymbol{\widetilde{\pi}} ^{i_{1:m-1}}}{ \boldsymbol{\pi} ^{i_{1:m-1}}}$, the second inequality is due to Holder's inequality, and the last inequality is by the triangular inequality in proposition \ref{prop:triangle}.

    Additionally, we notice that by repeating the same arguments, it still holds true that
    \begin{equation}
        \left| \mathbb E_{ {{\boldsymbol{\widetilde{\pi}}}}^{1:m-1}} A^{1:m}_{\boldsymbol{\pi}} 
        - 
        \mathbb E_{ \boldsymbol{\pi}^{1:m-1}}\min(\Re^{1:m-1}, 1)A^{1:m}_{\boldsymbol{\pi}}\right|
    \leq \bar \epsilon\sum_{i=1}^{m-1}\|\widetilde{\pi}^i -\pi^i\|_1. \label{eq:usefulboundimprv}
    \end{equation}
    The latter inequality will be useful in the improvement bound proof.
\end{proof}

    

The update rules described in the Algorithm \ref{alg:main} is 
\begin{flalign}
    \pi_{k+1}^{m}&=\argmax_{\pi^{m}} \big[L^{{1:m}, \text{clip}}_{\boldsymbol{\pi}_k}(\boldsymbol{{\pi}}_{k+1}^{{1:m-1}}, {\pi}^{m})-
    \beta_k
    \mathbb E_{d_{\rho}^{\boldsymbol{\pi}_k}}{D}_{\text{KL}}( \pi^{m}\|\pi_k^{m})\big] \notag \\
    &=\begin{multlined}[t]
        \argmax_{\pi^{m}} \mathbb E_{d_{\rho}^{\boldsymbol{\pi}_k}, \boldsymbol{a}^{1:m-1}\sim \boldsymbol{\pi}_k} \big[
    \min(\Re^{1:m-1}_{\boldsymbol{\pi}_{k+1}, \boldsymbol{\pi}_k}, 1)
\langle A^{1:m}_{\boldsymbol{\pi}_{k}}(s, \boldsymbol{a}^{1:m-1}, \cdot), \pi^m\rangle \\
-\beta_k D_\text{KL}(\pi^m \| \pi^m_k)
    \big]
    \end{multlined}
     \notag \\
    &= \begin{multlined}[t]
    \argmax_{\pi^m} \mathbb E_{d_{\rho}^{\boldsymbol{\pi}_k}, \boldsymbol{a}^{1:m-1}\sim \boldsymbol{\pi}_k}\big[ \min(\Re^{1:m-1}_{\boldsymbol{\pi}_{k+1}, \boldsymbol{\pi}_k}, 1)
\langle Q^{1:m}_{\boldsymbol{\pi}_{k}}(s, \boldsymbol{a}^{1:m-1}, \cdot), \pi^m\rangle \\
-\beta_k D_\text{KL}(\pi^m\| \pi^m_k)
\big]    
    \end{multlined}
    , \label{eq:clipobjective_appdx}
\end{flalign}
here we use the fact that $\Re^{1:m-1}_{\boldsymbol{\pi}_{k+1}, \boldsymbol{\pi}_k}$ is independent of the last $N-m+1$ actions, and that $V_{\boldsymbol{\pi}_k}(s)$ does not affect the optimal solution. The closed-form update of this optimization problem for any state $s$ with $d_{\rho}^{\boldsymbol{\pi}_k}(s)>0$ is
\begin{equation}
    \pi^m_{k+1}(\cdot|s) \propto \pi^m_{k}(\cdot|s) \exp\{ 
\beta_{k}^{-1}\mathbb E_{\boldsymbol{a}^{1:m-1}\sim \boldsymbol{\pi}_{k} } \min(\Re^{1:m-1}_k, 1)Q^{1:m}_{\boldsymbol{\pi}_k}(s, \boldsymbol{a}^{1:m-1}, \cdot) \label{eq:updaterule}
\}
\end{equation}
This is a well-known result in optimization, and the proof can be found, for example, in Example 3.71 \cite{beck2017first}.
\begin{remark}
    A remark on the distribution coefficient mismatch: Works on the convergence rate of Policy Gradient depend on a distribution mismatch coefficient. This term captures the random exploration of the policy in the form of the maximum ratio of the distributions (possibly a marginal state distribution with a restart model). In Natural Policy Gradient, \cite{agarwal2021theory} show the convergence independent of the mismatch coefficient, called dimension-free convergence. Natural Gradient and Trust Region Policy Optimization are known to have a close connection \cite{schulman2017equivalence}, where the update rule of the two methods appear to be the same (see Lemma 5.1 in \cite{agarwal2021theory}). 
    In Trust Region methods,
    we see that every state whose visitation probability is positive will be updated according to the update rule (\ref{eq:updaterule}), and that a stochastic policy can visit any reachable states from the initial distribution $\rho$ with non-zero probabilities. This is  stated below in the form of the finite concentration coefficient
\end{remark}

\begin{proposition} \label{prop:finitemismatch}
    For any $k\in \mathbb N$, if we restrict our states such that $s \in \mathcal S$ if it is reachable from $\rho$, then the distribution mismatch coefficient of the  is always finite, i.e. there exists a finite $D\in \mathbb R$ such that
    \begin{equation*}
        D= \max_{\boldsymbol{\pi}}\left\|\frac{d^{\boldsymbol{\pi}}_{\rho}}{d^{\boldsymbol{\pi}_k}_{\rho}}\right\|_\infty
    \end{equation*}
\end{proposition}
\begin{proof}
    At the iteration $k=0$, the initial joint policy is initialized to be a uniform distribution. As a result, since the state space is finite, any states that are reachable through a finite number of steps have a non-zero probability of being visited by the random policy $\boldsymbol{\pi}_0$. So if a state is reachable under $\rho$, it is also reachable using $\boldsymbol{\pi}_0$. Consequently, the update rule in (\ref{eq:updaterule}) is applied at each reachable state, and $d_{\rho}^{\boldsymbol{\pi}_0} > 0, \forall s \in \mathcal S$ if we restrict the state space to be the reachable states from $\rho$. 

    For an arbitrary step $k\in \mathbb N$, by looking at the update rule (\ref{eq:updaterule}),
    we see that the joint policy $\boldsymbol{\pi}_{k+1}$ is still a stochastic policy with $\boldsymbol{\pi}_{k+1}(\boldsymbol{a}|s)>0$ for all actions $\boldsymbol{a}$ in each state $s$, given the inductive assumption on $\boldsymbol{\pi}_{k}(\boldsymbol{a}|s)>0$. Thus the same arguments for when $k=0$ apply.

    Finally, for any $k \in \mathbb N$, Since $d_{\rho}^{\boldsymbol{\pi}_0} > 0, \forall s \in \mathcal S$, then the distribution mismatch coefficient is finite.
\end{proof}
This proposition ensures that the policies at each iteration of the update procedure are sufficiently explorative. As a result, since $d_{\rho}^{\boldsymbol{\pi}_k} > 0$ component-wise, minimizing the objective in (\ref{eq:clipobjective_appdx}) is equivalent to the update rule (\ref{eq:updaterule}) applied to all the states $s\in \mathcal S$ independently. In the single-agent Trust Region method, an equivalent condition to the Proposition \ref{prop:finitemismatch} can be seen in the assumption of the finite concentration coefficient in \cite{shani2020adaptive}. Due to the connection to Natural Gradient, Trust Region methods also enjoy dimension-free convergence when no error (approximation and estimation) is present in the learning process \cite{shani2020adaptive, zhao2023local}. 

\begin{proposition} \label{prop:boundedl1}
    For any two consecutive policies $\pi^m_k$ and $\pi^m_{k+1}$, then
    \[\|\pi^m_k - \pi^m_{k+1}\|_1 \leq \bar \epsilon \beta_k^{-1}\]
\end{proposition}
\begin{proof}
\begin{align*}
    & \|\pi^m_{k+1} - \pi^m_k\|^2_1\\
    \leq & D_\text{KL}(\pi^m_{k+1}\|\pi^m_{k}) + D_\text{KL}(\pi^m_{k}\|\pi^m_{k+1}) \tag{By Pinsker's inequality}\\
    =& \left\langle
    \log \frac{\pi^m_{k+1}}{\pi^m_{k}}, \pi^m_{k+1} - \pi^m_k
    \right\rangle\\
    =& \left\langle 
    \beta_{k}^{-1} \mathbb E_{\boldsymbol{\pi}^{1:m-1}_{k}} \min(\Re^{1:m-1}, 1) Q^{1:m}_{\boldsymbol{\pi}_k}, \pi^m_{k+1} - \pi^m_k 
    \right\rangle\\
    =& \left\langle 
    \beta_{k}^{-1} \mathbb E_{\boldsymbol{\pi}^{1:m-1}_{k}} \min(\Re^{1:m-1}, 1) A^{1:m}_{\boldsymbol{\pi}_k}, \pi^m_{k+1} - \pi^m_k 
    \right\rangle\\
    \leq& \beta_k^{-1}\bar \epsilon \|\pi^m_{k+1} - \pi^m_k \|_1  \tag{By Hölder's inequality}\\
    \Longrightarrow & \|\pi^m_{k+1} - \pi^m_k \|_1 \leq \beta_k^{-1}\bar \epsilon
\end{align*}
\end{proof}

\begin{proof}[Proof of Lemma \ref{lemma:onestep}]
In this proof, we omit $(\cdot|s)$ when it is clear from the context, and we write $Q^{1:m-1}_{\boldsymbol{\pi}}$ and $A^{1:m-1}_{\boldsymbol{\pi}}$ as shorthand for $Q^{1:m-1}_{\boldsymbol{\pi}}(s, \boldsymbol{a}^{1:m-1}, \cdot)$ and $A^{1:m-1}_{\boldsymbol{\pi}}(s, \boldsymbol{a}^{1:m-1}, \cdot)$, respectively,
\begin{align}
    &D_{\text{KL}}\left(\pi^m(\cdot | s)\| \pi^m_{k}(\cdot|s)\right) - D_{\text{KL}}\left(\pi^m(\cdot | s)\| \pi^m_{k+1}(\cdot|s)\right) \notag \\
    =& \left\langle 
    \log \frac{\pi^m_{k+1}}{\pi^m_{k}}, \pi^m
    \right\rangle \notag \\
    =& \left\langle
    \log \frac{\pi^m_{k+1}}{\pi^m_{k}}, \pi^m - \pi^m_{k+1}
    \right\rangle + D_{\text{KL}}\left(\pi_{k+1}^m\| \pi^m_{k}\right) \notag \\
    =& 
    \left\langle 
    \beta_{k}^{-1}\mathbb E_{\boldsymbol{\pi}_k^{1:m-1}} \min(\Re^{1:m-1}, 1)
    Q^{1:m}_{\boldsymbol{\pi}_k}, \pi^m - \pi^m_{k+1}
    \right\rangle + D_{\text{KL}}\left(\pi_{k+1}^m\| \pi^m_{k}\right) \notag \\
    \geq &  \left\langle 
    \beta_{k}^{-1} \mathbb E_{\boldsymbol{\pi}_k^{1:m-1}} \min(\Re^{1:m-1}, 1) Q^{1:m}_{\boldsymbol{\pi}_k}, \pi^m - \pi^m_{k+1}
    \right\rangle \label{eq:needforconvergence} \notag \\
    =& \left\langle
    \beta_{k}^{-1} \mathbb E_{\boldsymbol{\pi}_k^{1:m-1}} \min(\Re^{1:m-1}, 1) A^{1:m}_{\boldsymbol{\pi}_k}, \pi^m - \pi^m_{k+1}
    \right\rangle \notag \\
    \geq& \left\langle
    \beta_{k}^{-1} \mathbb E_{\boldsymbol{\pi}_{k}^{1:m-1}} A^{1:m}_{\boldsymbol{\pi}_k}, \pi^m - \pi^m_{k+1}
    \right\rangle 
    - 2\beta_k^{-1} \bar \epsilon\sum_{i=1}^{m-1} \|\pi_{k}^i -  {\pi}_{k+1}^i\|_1 \\
    \geq& \left\langle
    \beta_{k}^{-1} \mathbb E_{\boldsymbol{\pi}_{k}^{1:m-1}} A^{1:m}_{\boldsymbol{\pi}_k}, \pi^m - \pi^m_{k}
    \right\rangle 
    - 2\beta_k^{-1} \bar \epsilon\sum_{i=1}^{m} \|\pi_{k}^i -  {\pi}_{k+1}^i\|_1\\
    =& \left\langle
    \beta_{k}^{-1} \mathbb E_{\boldsymbol{\pi}_{k}^{1:m-1}} Q^{1:m}_{\boldsymbol{\pi}_k}, \pi^m - \pi^m_{k}
    \right\rangle 
    - 2\beta_k^{-1} \bar \epsilon\sum_{i=1}^{m} \|\pi_{k}^i -  {\pi}_{k+1}^i\|_1\\
    \geq & \left\langle
    \beta_{k}^{-1} \mathbb E_{\boldsymbol{\pi}_{k}^{1:m-1}} Q^{1:m}_{\boldsymbol{\pi}_k}, \pi^m - \pi^m_{k}
    \right\rangle 
    \begin{multlined}[t]
        - 2(\beta_k^{-1} \bar \epsilon)^2 m
    \end{multlined} 
\end{align}
The second inequality is derived from Lemma \ref{lemma:advantage_dif}, and the last inequality is from Proposition \ref{prop:boundedl1}. 
\end{proof}

\begin{proof}[Proof of Theorem {\upshape\ref{theorem:convergence}}]
First, we assume that the update order of the agents is fixed. Then, with Lemma \ref{lemma:onestep}, for any policy $\pi^m\in \Pi_m$ and any state $s \in \mathcal S$    
\begin{align*}
    &\beta_k^{-1}
    \left\langle \mathbb E_{\boldsymbol{\pi}_{k}^{1:m-1}} Q^{1:m}_{\boldsymbol{\pi}_k}, \pi^m - \pi^m_{k}
    \right\rangle \\
    \leq & 
    \bigg[ D_\text{KL}\left( \pi^m \| \pi^m_{k} \right) - D_\text{KL}\left( \pi^m \| \pi^m_{k+1} \right) 
    \begin{multlined}[t]
    + 2(\beta^{-1}\bar \epsilon)^2m
    \bigg]
    \end{multlined}
\end{align*}
Summing over $k=0, \dots, K-1$
\begin{flalign*}
    &\sum_k^{K-1}\beta_k^{-1} \left\langle \mathbb E_{\boldsymbol{\pi}_{k}^{1:m-1}} Q^{1:m}_{\boldsymbol{\pi}_k}, \pi^m - \pi^m_{k}
    \right\rangle\\
    \leq & D_\text{KL}\left( \pi^m \| \pi^m_{0} \right)
    - D_\text{KL}\left( \pi^m \| \pi^m_{K} \right)
    \begin{multlined}[t]
    + \sum_{k=0}^{K-1}
     2m(\beta^{-1}\bar {\epsilon})^2
    \end{multlined}\\
    \leq & 
    \log A + 2\sum_{k=0}^{K-1}\frac{N\bar {\epsilon}^2}{\beta_k^2}
\end{flalign*}
By setting $\beta_k = \beta \sqrt{K}$, then we have
\begin{align}
    \frac{1}{K}\sum_{k=0}^{K-1}\left\langle \mathbb E_{\boldsymbol{\pi}_{k}^{1:m-1}} Q^{1:m}_{\boldsymbol{\pi}_{k}}, \pi^m - \pi^m_{k}
    \right\rangle \leq \frac{\beta\log A + 2\frac{N\bar {\epsilon}^2 }{\beta}}{\sqrt{K}}
    \label{eq:rawconvergence}
\end{align}
So with the optimal $\beta=\sqrt{\frac{2N\bar {\epsilon}^2}{\log A}}$, 
\begin{align}
    \frac{1}{K}\sum_{k=0}^{K-1} \left\langle \mathbb E_{\boldsymbol{\pi}_{k}^{1:m-1}} Q^{1:m}_{\boldsymbol{\pi}_k}, \pi^m - \pi^m_{k}
    \right\rangle \leq 
    \frac{2\bar {\epsilon}\sqrt{2N\log A }}{\sqrt{K}}
    \label{eq:rawconvergence_refined}
\end{align}
Also note that this is true for all $\pi^m\in\Pi^m, \forall m$, and $\forall s \in \mathcal S$, so
\begin{align*}
    \frac{1}{K}\sum_{k=0}^{K-1} \mathbb E_{d_{\rho}^{(\boldsymbol{\pi}_k^{-m}, \pi^m)}}\left\langle \mathbb E_{\boldsymbol{\pi}_{k}^{1:m-1}} Q^{1:m}_{\boldsymbol{\pi}_k}, \pi^m - \pi^m_{k}
    \right\rangle \leq 
    \frac{2\bar {\epsilon}\sqrt{2N\log A }}{\sqrt{K}}
\end{align*}
\[\Longrightarrow \frac{1}{K}\sum_{k=0}^{K-1} \left( J(\boldsymbol{\pi}_k^{-m}, \pi^m) -J(\boldsymbol{\pi}_k)\right) 
\leq
\frac{2\bar {\epsilon}\sqrt{2N\log A }}{(1-\gamma)\sqrt{K}} \tag{By Difference lemma \ref{lemma:difference}}
\]
And since $\boldsymbol{\bar \pi}$ is uniformly sampled from $\boldsymbol{\pi}_k$, with $k = 0, 1, \dots, K-1$, then we have
\begin{equation}
    J(\boldsymbol{\bar \pi}^{-m}, \pi^m) -J(\boldsymbol{\bar \pi})\leq 
O\left(
\frac{\bar {\epsilon}\sqrt{N\log A }}{(1-\gamma)\sqrt{K}}
\right) \label{eq:convergence_rate}
\end{equation}
Note that this is true for any $\pi^m \in \Pi^m$ and any agent $m$, then $\boldsymbol{\bar \pi}$ is an $\varepsilon$-Nash equilibrium by definition. Furthermore, this result is still true when the order of agents is permuted at each update cycle, as long as we match agents correctly between consecutive iterations. 
\end{proof}

\subsection{Omitted results of Policy Improvement Bound} \label{proof:policyimprovement}

\begin{lemma}\label{lem:trpotakenadvantagebound}
    For any joint policies $\boldsymbol{\widetilde \pi}$ and $\boldsymbol{\pi}$, then the following inequality holds
    \begin{flalign*}
        \left|\mathbb E_{s_t \sim \boldsymbol{\widetilde\pi}, \boldsymbol{a}\sim \boldsymbol{\widetilde\pi}}A_{\boldsymbol{\pi}}(s_t, \boldsymbol{a}) -
        \mathbb E_{s_t\sim\boldsymbol{\pi}, \boldsymbol{a}\sim \boldsymbol{\widetilde\pi}}A_{\boldsymbol{\pi}}(s_t, \boldsymbol{a}) \right|
        \begin{multlined}[t]
            \leq 4\epsilon \alpha \left(1-(1-\alpha)^t\right)
        \end{multlined}
    \end{flalign*}
    where we denote $\alpha=\max_s \|\boldsymbol{\pi} - \boldsymbol{\widetilde\pi}\|_1$, and $\epsilon=\max_{s, \boldsymbol{a}}A_{\boldsymbol{\pi}}$.
\end{lemma}
\begin{proof}
    Refer to the proof of Lemma 3 in  \cite{schulman2015trust}, or Lemma 3 in \cite{wang2023order}.
\end{proof}

\begin{lemma}[Multi-Agent Advantage Decomposition]\label{lemma:multagentdecomposition}
    In cooperative Markov games, for any joint policy $\boldsymbol{\pi}$, state $s$ and joint action $\boldsymbol{a}$, the following decomposition holds
    \[A_{\boldsymbol{\pi}}(s, \boldsymbol{a})=\sum_{m=1}^N A_{\boldsymbol{\pi}}^m(s, \boldsymbol{a}^{1:m-1}, a^m).\]
\end{lemma}
\begin{proof}
    Refer to Lemma 1 in \cite{kuba2021trust}.
\end{proof}

\begin{proof} [Proof of Lemma {\upshape\ref{lem:improvementboudclip}}]
Here we write $A_{\boldsymbol{\pi}}$ as shorthand for $A_{\boldsymbol{\pi}}(s, \boldsymbol{a})$
    \begin{flalign*}
        &\bigg|J(\boldsymbol{\widetilde \pi}) - 
        J(\boldsymbol{\pi}) - \frac{1}{1-\gamma}\sum_{m=1}^N L^{\mathrm{1:m, clip}}_{\boldsymbol{\pi}}(\boldsymbol{\widetilde{\pi}}^{1:m-1}, \widetilde{\pi}^m)\bigg|\\
        =&\begin{multlined}[t]
            \frac{1}{1-\gamma} \bigg| 
        \mathbb E_{s\sim d_{\rho}^{\boldsymbol{\widetilde\pi}}}
        \mathbb E_{\boldsymbol{a}\sim \boldsymbol{\widetilde\pi}}A_
        {\boldsymbol{\pi}}
        - \sum_{m=1}^N
        \mathbb E_{s\sim d_{\rho}^{\boldsymbol{\pi}}, \boldsymbol{a}\sim\boldsymbol{\pi}}
        \bigg(
        \frac{\widetilde{\pi}^m}{\pi^m}-1
        \bigg)
        \min(\Re^{1:m-1}, 1)A_{\boldsymbol{\pi}}\bigg|
        \end{multlined}\\
        \leq &\begin{multlined}[t]
            \frac{1}{1-\gamma} \bigg| 
        \mathbb E_{s\sim d_{\rho}^{\boldsymbol{\widetilde\pi}}}
        \mathbb E_{\boldsymbol{a}\sim \boldsymbol{\widetilde\pi}}A_
        {\boldsymbol{\pi}}(s, \boldsymbol{a}) 
        - \sum_{m=1}^N
        \mathbb E_{s\sim d_{\rho}^{\boldsymbol{\pi}}, \boldsymbol{a}\sim\boldsymbol{\pi}}
        \bigg(
        \frac{\widetilde{\pi}^m}{\pi^m}-1
        \bigg)
        \Re^{1:m-1}A_{\boldsymbol{\pi}}\bigg| \\
        + \frac{1}{1-\gamma}\sum_{m=1}^N\mathbb E_{s\sim d_{\rho}^{\boldsymbol{\pi}}, a^m\sim\pi^m}\epsilon
        \bigg|
        \frac{\widetilde{\pi}^m}{\pi^m}-1
        \bigg|
        \sum_{i=1}^{m-1} \alpha^i \tag{By \ref{eq:usefulboundimprv}}
        \end{multlined}\\
        = &\frac{1}{1-\gamma} \bigg| 
        \mathbb E_{s\sim d_{\rho}^{\boldsymbol{\widetilde\pi}}}
        \mathbb E_{\boldsymbol{a}\sim \boldsymbol{\widetilde\pi}}
        A_
        {\boldsymbol{\pi}} - 
        \begin{multlined}[t]
        \mathbb E_{s\sim d_{\rho}^{\boldsymbol{\pi}}}
        \sum_{m=1}^N
        \mathbb E_{\boldsymbol{a}^{1:m}\sim \boldsymbol{\widetilde\pi}}
            A^m_
        {\boldsymbol{\pi}}
        \bigg|
        + \frac{\epsilon}{1-\gamma}\sum_{m=1}^N \alpha^m \sum_{i=1}^{m-1} \alpha^i 
        \end{multlined} \tag{By \ref{eq:advest}} \\
        = &\bigg| 
        \frac{1}{1-\gamma}\bigg(\mathbb E_{s\sim d_{\rho}^{\boldsymbol{\widetilde\pi}}}
        \mathbb E_{\boldsymbol{a}\sim \boldsymbol{\widetilde\pi}}
        A_
        {\boldsymbol{\pi}} - 
        \begin{multlined}[t]
        \mathbb E_{s\sim d_{\rho}^{\boldsymbol{\pi}}}
        \mathbb E_{\boldsymbol{a}\sim \boldsymbol{\widetilde\pi}}
            A_
        {\boldsymbol{\pi}} \bigg)
        \bigg| 
        + \frac{\epsilon}{1-\gamma}\sum_{m=1}^N \alpha^m \sum_{i=1}^{m-1} \alpha^i 
        \end{multlined} \tag{By Lemma \ref{lemma:multagentdecomposition}}\\
        \leq & {4\epsilon} \bigg| \mathbb E  
        \sum_{t=0}^{\infty} \gamma^t \alpha \left(1-(1-\alpha)^t\right)
        \bigg| + \frac{\epsilon}{1-\gamma}\sum_{m=1}^N \alpha^m \sum_{i=1}^{m-1} \alpha^i \tag{By Lemma \ref{lem:trpotakenadvantagebound}} \\
        \leq & \frac{4\gamma\epsilon}{(1-\gamma)^2}\alpha^2 + \frac{\epsilon}{1-\gamma}\sum_{m=1}^N \alpha^m \sum_{i=1}^{m-1} \alpha^i \\
        \leq & \frac{4\gamma\epsilon}{(1-\gamma)^2}\bigg(\sum_{m=1}^N\alpha^m\bigg)^2
        + \frac{\epsilon}{1-\gamma}\sum_{m=1}^N \alpha^m \sum_{i=1}^{m-1} \alpha^i 
    \end{flalign*}
\end{proof}

\subsection{Omitted Proof of Variance Explosion}\label{sec:varianceproof}

\begin{proof}[Proof of Proposition \ref{prop:exponentialexample}]
    As there is only one state, we omit the state in the notations for brevity. It is easy to see that, initially, the joint value function $V_{\boldsymbol{\pi}} = 0$ and the advantage $A_{\boldsymbol{\pi}}(\boldsymbol{a}) \in \{-1, 1\}, \forall \boldsymbol{a}$. Let $E$ be the event that at least half of the agents take actions 1 and $\overline{E}$ its complement, then since all agents act randomly uniform, $\mathbb P(E) = \mathbb P(\overline{E}) = 0.5$. Furthermore, we have $A_{\boldsymbol{\pi}}(\boldsymbol{a}) | E = 1$ and $A_{\boldsymbol{\pi}}(\boldsymbol{a}) | \overline{E} = -1$.

    First, we claim that if $A_{\boldsymbol{\pi}}(\boldsymbol{a})=1$, then $\frac{\widetilde{\pi}^{m}(a^{m})}{\pi^{m}(a^{m})} \geq 1.5 , \forall m \leq 2N$. The proof of the proposition then follows easily, as
    \begin{align*}
        \mathbb E_{\boldsymbol{a}^{{1:m-1}} \sim \boldsymbol{\widetilde{\pi}}} 
        \Re^{{1:m-1}}_{\boldsymbol{\widetilde{\pi}}, \boldsymbol{\pi}} &\geq 
        \mathbb P(E)\mathbb E_{\boldsymbol{a}^{{1:m-1}} \sim \boldsymbol{\widetilde{\pi}}} 
        \big[\Re^{{1:m-1}}_{\boldsymbol{\widetilde{\pi}}, \boldsymbol{\pi}}(\boldsymbol{a}) | E\big] \\
        &= \mathbb P(E)\mathbb E_{\boldsymbol{a}^{{1:m-1}} \sim \boldsymbol{\widetilde{\pi}}} 
        \bigg[\prod_{k=1}^{m-1} \frac{\widetilde{\pi}^{k}(a^{k})}{\pi^{k}(a^{k})}\bigg | E\bigg]
        \geq \frac{1.5^{m-1}}{2}.
    \end{align*}

    We will prove the above claim by induction. 

    At the first agent update iteration, i.e. $m=1$, we have that $A_{\boldsymbol{\pi}}(\boldsymbol{a}) = 1$ from the assumption of the claim, and $\Re^{{1:m-1}}_{\boldsymbol{\widetilde{\pi}}, \boldsymbol{\pi}}(\boldsymbol{a})=1$ by definition. 
    Differentiate the objective w.r.t. the current policy $\hat\pi^{m}(a^m)$ and set it to 0,
    \begin{align*}
        0 &= \nabla_{\hat\pi^{m}} \left[
        \hat L^{{1:m}}_{\boldsymbol{\pi}} (\boldsymbol{\widetilde{\pi}}^{{1:m-1}}, \hat{\pi}^{m} )- 
        C
        \text{D}_{\text{KL}}^{\max}(\pi^{m}, \hat{\pi}^{m}) \right] \\
        &= \begin{multlined}[t]
            \frac{\nabla_{\hat{\pi}^{m}}\hat{\pi}^{m}(a^{m})}{\pi^{m}(a^{m})}
    \Re^{{1:0}}_{\boldsymbol{\widetilde{\pi}}, \boldsymbol{\pi}}(\boldsymbol{a}) A_{\boldsymbol{\pi}}(\boldsymbol{a}) 
    + C\sum_{a\in \{0, 1\}} \pi^{m}(a) 
    \nabla_{\hat{\pi}^{m}}\log\frac{\hat\pi^{m}(a)}{\pi^{m}(a)}
        \end{multlined}\\
    &= \begin{multlined}[t]
            \frac{1}{\pi^{m}(a^{m})}
    \Re^{{1:0}}_{\boldsymbol{\widetilde{\pi}}, \boldsymbol{\pi}}( \boldsymbol{a}) A_{\boldsymbol{\pi}}(\boldsymbol{a}) 
    + C \left[
    \frac{\pi^{m}(a^{m})}{\hat\pi^{m}(a^{m})}
    -\frac{1-\pi^{m}(a^{m})}{1-\hat\pi^{m}(a^{m})}
    \right]
        \end{multlined}\\
    \Rightarrow 0&=
    \begin{multlined}[t]
        - \hat\pi^{m}(a^{m})^2 
    \Re^{{1:0}}_{\boldsymbol{\widetilde{\pi}}, \boldsymbol{\pi}}( \boldsymbol{a})
    A_{\boldsymbol{\pi}}(\boldsymbol{a}) 
    + \hat\pi^{m}(a^{m}) [
    \Re^{{1:0}}_{\boldsymbol{\widetilde{\pi}}, \boldsymbol{\pi}}( \boldsymbol{a})
    A_{\boldsymbol{\pi}}(\boldsymbol{a}) - C\pi^{m}(a^{m})] + C\pi^{m}(a^{m})^2
    \end{multlined}
    \end{align*}
which is a quadratic equation of $\hat\pi^{m}(a^{m})$, plugging $A_{\boldsymbol{\pi}}(\boldsymbol{a})=1, \Re^{{1:0}}_{\boldsymbol{\widetilde{\pi}}, \boldsymbol{\pi}}( \boldsymbol{a})=1, \pi^{1}=1/2$ and $C=3/2$, we yield $\widetilde{\pi}^{1}(a^{1})=3/4$. As a result, $\frac{\widetilde{\pi}^{1}(a^{1})}{\pi^{1}(a^{1})}=3/2$.

At some iteration $m> 1$, suppose that $\widetilde{\pi}^{k}(a^{k}) \geq 3/4, \forall k < m$ (by previous induction steps), we will prove that $\widetilde{\pi}^{m}(a^{m}) \geq 3/4$, which leads to $\frac{\widetilde{\pi}^{m}(a^{m})}{\pi^{m}(a^{m})} \geq 3/2$.

For any $0 < p < 1$, define 
\begin{flalign*}
    \begin{multlined}[t]
        \mathcal L^m(p) = \frac{p}{\pi^m(a^m)}\Re^{{1:m-1}}_{\boldsymbol{\widetilde{\pi}}, \boldsymbol{\pi}}( \boldsymbol{a})A_{\boldsymbol{\pi}}(\boldsymbol{a}) - C \bigg[\pi^m(a^m)\log\frac{\pi^m(a^m)}{p} + \big(1-\pi^m(a^m)\big) \\
        \log\frac{1-\pi^m(a^m)}{1-p}\bigg].
    \end{multlined}
\end{flalign*}
That is, we define $p$ as the parameter of the distribution to be optimized for the sequential update. Then using the fact that $\pi^m(a^m)=\pi^{m-1}(a^{m-1})=1/2$, we have the following identity
\begin{equation}
    \mathcal L^{m}(p) - \mathcal L^{m-1}(p) = p \underbrace{2A_{\boldsymbol{\pi}}(\boldsymbol{a})
    \big( \Re^{{1:m-1}}_{\boldsymbol{\widetilde{\pi}}, \boldsymbol{\pi}}( \boldsymbol{a}) - \Re^{{1:m-2}}_{\boldsymbol{\widetilde{\pi}}, \boldsymbol{\pi}}( \boldsymbol{a})}_{(*)}
    \big), \notag
\end{equation}
which is a linear function of $p$. Furthermore, since $\Re^{{1:m-1}}_{\boldsymbol{\widetilde{\pi}}, \boldsymbol{\pi}}( \boldsymbol{a}) = \frac{\widetilde{\pi}^{m-1}}{{\pi}^{m-1}}\Re^{{1:m-2}}_{\boldsymbol{\widetilde{\pi}}, \boldsymbol{\pi}}( \boldsymbol{a})\geq 3/2 \Re^{{1:m-2}}_{\boldsymbol{\widetilde{\pi}}, \boldsymbol{\pi}}( \boldsymbol{a})$ by induction steps, and combine with $A_{\boldsymbol{\pi}}(\boldsymbol{a})=1$, then we have $(*)>0$. 

Define $p^{m-1}_{*}=\argmax_p \mathcal L^{m-1}(p)$ the optimal solution of the sequential update in the previous step (we can see that $p^{m-1}_{*}$ is well-defined), or in other words, $p^{m-1}_{*}=\widetilde{\pi}^{m-1}(a^{m-1})$. Now consider an arbitrary $p'$ such that $0 < p' \leq p^{m-1}_{*}$, then
\begin{flalign*}
    \mathcal L^{m}(p^{m-1}_{*}) - \mathcal L^{m}(p') &= \underbrace{\mathcal L^{m-1}(p^{m-1}_{*}) - \mathcal L^{m-1}(p')}_{\geq 0} + \underbrace{(*) [p^{m-1}_{*} - p']}_{\geq 0}
    \geq 0,
\end{flalign*}
here the first inequality is due to the optimality of $p^{m-1}_{*}$, and the second inequality is from the definition of $p'$. As a result, if we define $\widetilde{\pi}^m(a^m) = \argmax_{p}\mathcal L^{m}(p)$, then it must be true that $\widetilde{\pi}^m({a}^m) \geq p^{m-1}_{*}=\widetilde{\pi}^{m-1}(a^{m-1})$. Repeat this argument for the induction steps, we see that $\widetilde{\pi}^m({a}^m)\geq \widetilde{\pi}^1({a}^1)=3/4$.
\end{proof}

\section{Experimental Details}
\label{appendix:experimentdetails}

Table \ref{table:commonhyper} presents all the common hyperparameters used in our experiments, which are adopted from \cite{zhong2023heterogeneous}. Different hyperparameters for each evaluated benchmark are listed in the subsequent tables \ref{tab:cliphappomamujocohyper}, \ref{tab:cliphatrpomamujocohyper}, \ref{tab:cliphapposmachyper}, \ref{tab:cliphatrposmachyper}, and \ref{tab:mpehyper}.

\begin{table}[htbp]
\caption{Common hyperparameter settings across all environments.}
\label{table:commonhyper} 
\centering
\catcode`,=\active
\def,{\char`,\allowbreak}
\renewcommand\arraystretch{1.2}
\begin{tabular}{p{4cm}<{\raggedright} p{1.5cm} | p{4cm}<{\raggedright} p{1.5cm} }
\toprule
   \textbf{Hyperparameters}             & \textbf{Values}    & \textbf{Hyperparameters}             & \textbf{Values}      \\ 
  \midrule
    activation & ReLU & use proper time limits & True \\
    use feature normalization & True  & initialization method & orthogonal                     \\
   use naive recurrent policy & False &
    gain               & 0.01    \\
    num GRU layers & 1 & data chunk length & 10 \\
    optim eps & 1e-5 & weight decay & 0 \\
    std x coef & 1 & std y coef & 0.5 \\
    value loss coef & 1 & use max grad norm & True \\
    max grad norm & 10.0 & use GAE & True \\
    GAE lambda & 0.95 & use huber loss & True \\
    use policy active masks & True & huber delta  & 10.0\\
    ls step & 10 & accept ratio & 0.5 \\
    n training rollouts & 20 \\
  \botrule
\end{tabular} 
\end{table}

\begin{table}[]
\caption{Different hyperparameters used for clip-HAPPO in the MAMuJoCo domain.}
    \label{tab:cliphappomamujocohyper}
    \centering
    \begin{tabular}{ p{3.5cm}|p{1cm} p{1cm} p{1.cm} p{1cm} p{1.5cm} p{1cm} } 
 \toprule
 scenarios & linear lr decay & actor/ critic lr & ppo/ critic epoch & clip param & actor/ critic mini batch & entropy coef \\
 \hline
 Ant 4x2 & False & 5e-4 & 5 & 0.1 & 1 & 0 \\ 
 Ant 8x1  & False & 1e-3 & 5 & 0.1 & 1 & 0 \\ 
 HalfCheetah 2x3 & False & 5e-4 & 15 & 0.05 & 1 & 0.01 \\
 HalfCheetah 6x1 & False & 5e-4 & 15 & 0.05 & 1 & 0.01 \\
 Humanoid 17x1 & True & 5e-4 & 5 & 0.1 & 1 & 0 \\
 Walker2d 3x2 & True & 1e-3 & 5 & 0.05 & 2 & 0 \\
 Walker2d 6x1 & False & 1e-3 & 5 & 0.05 & 2 & 0 \\
 ManyagentSwimmer 10x2 & False & 1e-4 & 5 & 0.15 & 1 & 1e-6 \\
 ManyagentSwimmer 20x1 & False & 1e-4 & 5 & 0.1 & 1 & 1e-5 \\
 \botrule
\end{tabular}
\end{table}

\begin{table}[]
\caption{Different hyperparameters used for clip-HATRPO in the MAMuJoCo domain.}
    \label{tab:cliphatrpomamujocohyper}
    \centering
    \begin{tabular}{ p{3.5cm}|p{1cm} p{1cm} p{1cm} p{1cm} p{1.5cm} p{1cm}} 
 \toprule
 scenarios & linear lr decay & critic lr & critic epoch & clip param & critic mini batch & kl threshold \\
 \hline
 Ant 4x2 & False & 5e-4 & 5 & 0.2 & 1 & 5e-3 \\ 
 Ant 8x1  & False & 5e-4 & 5 & 0.2 & 1 & 5e-3 \\ 
 HalfCheetah 2x3 & False & 5e-4 & 5 & 0.2 & 1 & 5e-3 \\
 HalfCheetah 6x1 & False & 5e-4 & 5 & 0.2 & 1 & 1e-2 \\
 Humanoid 17x1 & False & 5e-4 & 5 & 0.2 & 1 & 1e-2 \\
 Walker2d 3x2 & False & 1e-3 & 5 & 0.2 & 1 & 1e-2 \\
 Walker2d 6x1 & False & 5e-4 & 5 & 0.2 & 1 & 5e-3 \\
 ManyagentSwimmer 10x2 & False & 1e-4 & 5 & 0.15 & 1 & 1e-2 \\
 ManyagentSwimmer 20x1 & False & 1e-4 & 5 & 0.15 & 1 & 1e-2 \\
 \botrule
\end{tabular}
\end{table}

\begin{table}[]
\caption{Different hyperparameters used for clip-HAPPO in the SMAC domain.}
    \label{tab:cliphapposmachyper}
    \centering
    \begin{tabular}{ p{2.5cm}|p{1cm} p{1.5cm} p{1cm} p{1.5cm} p{1.5cm} p{1cm}} 
 \toprule
 scenarios & network & ppo/critic epoch & clip param & actor/critic mini batch & gamma & state type \\
 \hline
 27m\_vs\_30m & RNN & 5 & 0.05 & 0.2 & 0.95 & FP \\ 
 3s5z\_vs\_3s6z  & RNN & 5 & 0.2 & 2 & 0.95 & FP \\ 
 corridor & MLP & 5 & 0.2 & 1 & 0.99 & FP \\
 \botrule
\end{tabular}
\end{table}

\begin{table}[]
\caption{Different hyperparameters used for clip-HATRPO in the SMAC domain.}
    \label{tab:cliphatrposmachyper}
    \centering
    \begin{tabular}{ p{2.5cm}|p{1cm} p{1.5cm} p{1cm} p{1.5cm} p{1.5cm} p{1cm}} 
 \toprule
 scenarios & network & critic lr & gamma & kl threshold & backtrack coeff & state type \\
 \hline
 27m\_vs\_30m & RNN & 5e-4 & 0.99 & 1e-3 & 0.8 & FP \\ 
 3s5z\_vs\_3s6z  & MLP & 5e-4 & 0.99 & 5e-3 & 0.5 & FP \\ 
 corridor & RNN & 5e-4 & 0.99 & 0.12 & 0.5 & FP \\
 \botrule
\end{tabular}
\end{table}

\begin{table}[]
\caption{Hyperparameter settings for MPE domain.}
\label{tab:mpehyper} 
\centering
\catcode`,=\active
\def,{\char`,\allowbreak}
\renewcommand\arraystretch{1.2}
\begin{tabular}{p{3cm}<{\raggedright} p{2.5cm} | p{2.5cm}<{\raggedright} p{1.5cm} }
\toprule
   \textbf{Hyperparameters}             & \textbf{Values}    & \textbf{Hyperparameters}             & \textbf{Values}      \\ 
  \midrule
    ppo/critic epoch & 5 (10 in \textit{reference}) & clip param & 0.2 \\
    entropy coef & 0.01  & backtrack coef & 0.8 \\
    linear lr decay & False &
    network  & MLP    \\
    actor/critic lr & 5e-4 & gamma & 0.99 \\
    actor/critic minibatch & 1 & kl threshold & 0.005 \\
  \botrule
\end{tabular} 
\end{table}




\section{{Additional Experiment Results}} \label{sec:c}
We present additional experiment results on the changing number of threads on 6 scenarios from the Mamujoco benchmark. For all experiments, we keep the hyperparameter the same as in the main experiments with MaMujoco, with only the training threads reduced. To demonstrate the effectiveness of our method under a low resource constraints regime, we test on a variety of 6 different scenarios, with the threads reduced to $1/2$ and $1/4$ of the usual values, which corresponds to using 10 and 5 threads during training. Fig. \ref{fig:mujoco_10appendix} {and Fig.} \ref{fig:mujoco_5appendix} {illustrate the performance of HAPPO and our clip-HAPPO. For comparison, our method better maintains the performance than the HAPPO in most settings. 
HAPPO only outperforms ours in the HalfCheetah-v2-6x1 with 5 threads settings.
In terms of performance, our clip HAPPO maintains stable performance across 4 of 6 tested environments when changing from 10 to 5 threads.
A lower value of threads, for example, 1 or 2, would hurt the learning performance as the estimation becomes unstable both algorithms fail to extract the useful signal.
On the other hand, the unstable issue could be balanced by adjusting other hyper-parameters such as the policy learning rate, target entropy parameters, or the number of PPO epochs.
}

\begin{figure*}[htbp]

\centering\includegraphics[width=1.0\columnwidth]{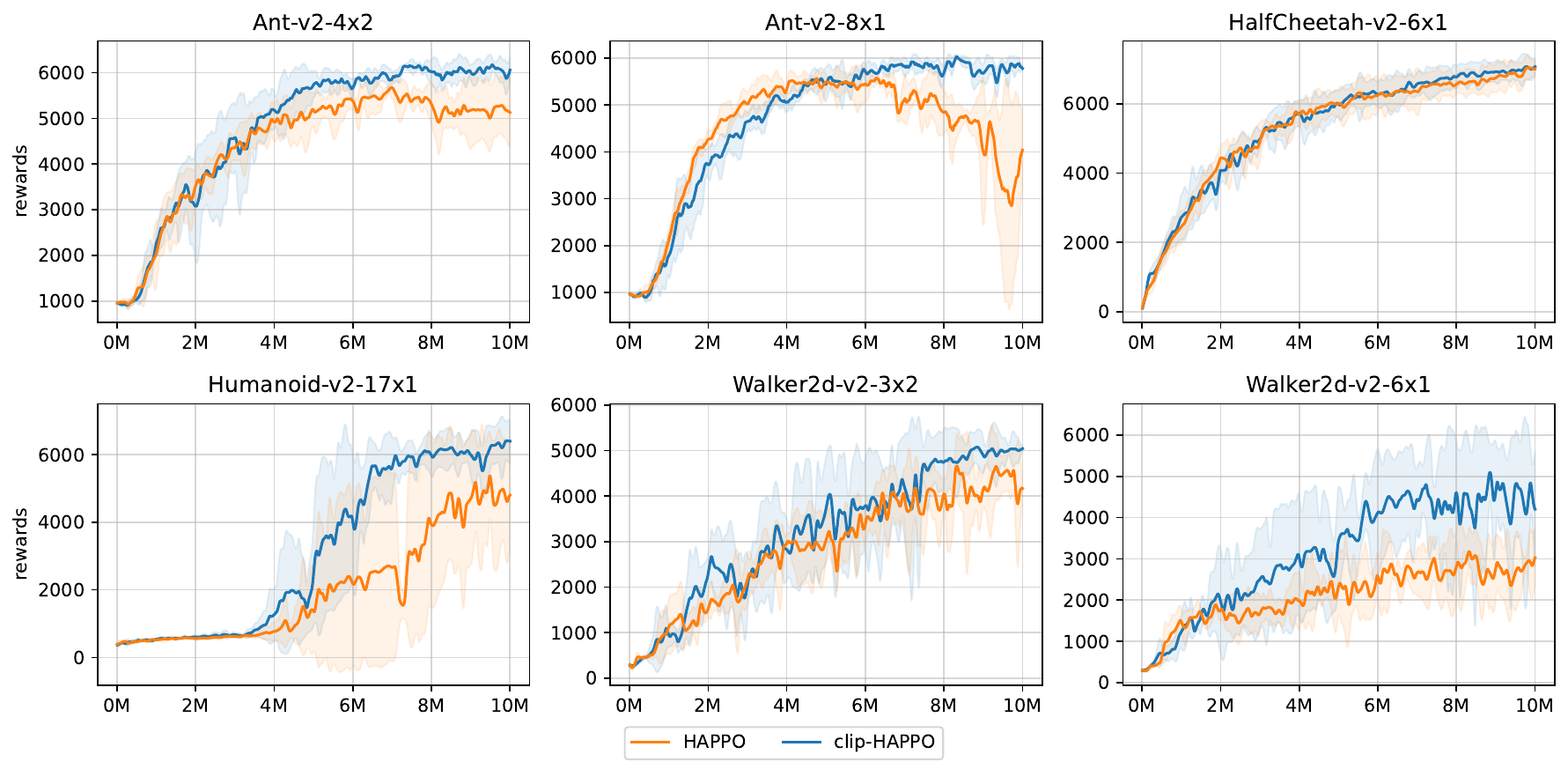}
  \caption{Results on 6 environments of Mamujoco with half the number of threads.}
  \label{fig:mujoco_10appendix}
\end{figure*}

\begin{figure*}[htbp]

\centering\includegraphics[width=1.0\columnwidth]{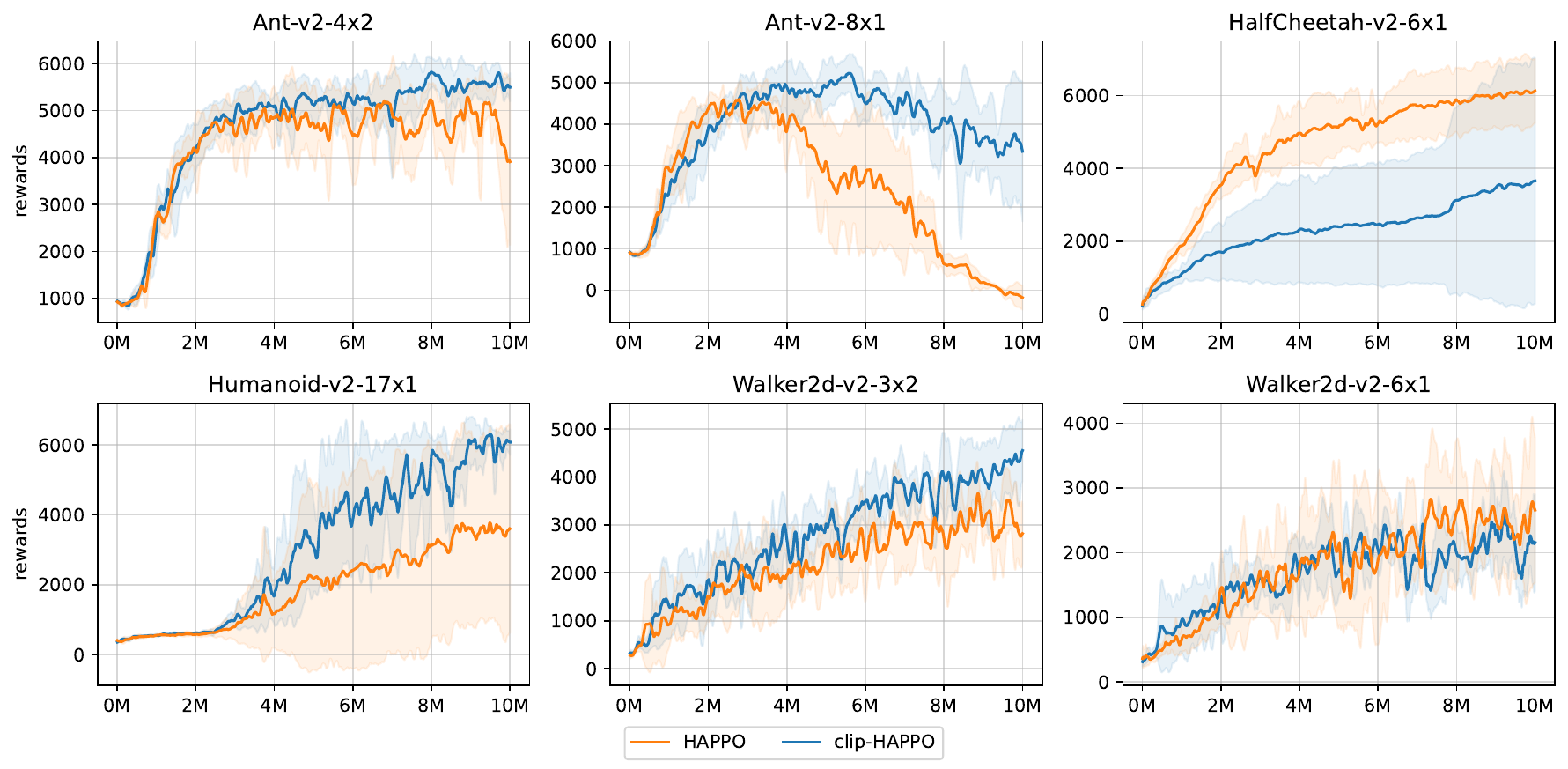}
  \caption{Results on 6 environments of Mamujoco with a quarter the number of threads.}
  \label{fig:mujoco_5appendix}
\end{figure*}

\end{appendices}


\clearpage

\begin{thebibliography}{42}
\ifx \bisbn   \undefined \def \bisbn  #1{ISBN #1}\fi
\ifx \binits  \undefined \def \binits#1{#1}\fi
\ifx \bauthor  \undefined \def \bauthor#1{#1}\fi
\ifx \batitle  \undefined \def \batitle#1{#1}\fi
\ifx \bjtitle  \undefined \def \bjtitle#1{#1}\fi
\ifx \bvolume  \undefined \def \bvolume#1{\textbf{#1}}\fi
\ifx \byear  \undefined \def \byear#1{#1}\fi
\ifx \bissue  \undefined \def \bissue#1{#1}\fi
\ifx \bfpage  \undefined \def \bfpage#1{#1}\fi
\ifx \blpage  \undefined \def \blpage #1{#1}\fi
\ifx \burl  \undefined \def \burl#1{\textsf{#1}}\fi
\ifx \doiurl  \undefined \def \doiurl#1{\url{https://doi.org/#1}}\fi
\ifx \betal  \undefined \def \betal{\textit{et al.}}\fi
\ifx \binstitute  \undefined \def \binstitute#1{#1}\fi
\ifx \binstitutionaled  \undefined \def \binstitutionaled#1{#1}\fi
\ifx \bctitle  \undefined \def \bctitle#1{#1}\fi
\ifx \beditor  \undefined \def \beditor#1{#1}\fi
\ifx \bpublisher  \undefined \def \bpublisher#1{#1}\fi
\ifx \bbtitle  \undefined \def \bbtitle#1{#1}\fi
\ifx \bedition  \undefined \def \bedition#1{#1}\fi
\ifx \bseriesno  \undefined \def \bseriesno#1{#1}\fi
\ifx \blocation  \undefined \def \blocation#1{#1}\fi
\ifx \bsertitle  \undefined \def \bsertitle#1{#1}\fi
\ifx \bsnm \undefined \def \bsnm#1{#1}\fi
\ifx \bsuffix \undefined \def \bsuffix#1{#1}\fi
\ifx \bparticle \undefined \def \bparticle#1{#1}\fi
\ifx \barticle \undefined \def \barticle#1{#1}\fi
\bibcommenthead
\ifx \bconfdate \undefined \def \bconfdate #1{#1}\fi
\ifx \botherref \undefined \def \botherref #1{#1}\fi
\ifx \url \undefined \def \url#1{\textsf{#1}}\fi
\ifx \bchapter \undefined \def \bchapter#1{#1}\fi
\ifx \bbook \undefined \def \bbook#1{#1}\fi
\ifx \bcomment \undefined \def \bcomment#1{#1}\fi
\ifx \oauthor \undefined \def \oauthor#1{#1}\fi
\ifx \citeauthoryear \undefined \def \citeauthoryear#1{#1}\fi
\ifx \endbibitem  \undefined \def \endbibitem {}\fi
\ifx \bconflocation  \undefined \def \bconflocation#1{#1}\fi
\ifx \arxivurl  \undefined \def \arxivurl#1{\textsf{#1}}\fi
\csname PreBibitemsHook\endcsname

\bibitem[\protect\citeauthoryear{Schulman et~al.}{2015}]{schulman2015trust}
\begin{bchapter}
\bauthor{\bsnm{Schulman}, \binits{J.}},
\bauthor{\bsnm{Levine}, \binits{S.}},
\bauthor{\bsnm{Abbeel}, \binits{P.}},
\bauthor{\bsnm{Jordan}, \binits{M.}},
\bauthor{\bsnm{Moritz}, \binits{P.}}:
\bctitle{Trust region policy optimization}.
In: \bbtitle{International Conference on Machine Learning},
pp. \bfpage{1889}--\blpage{1897}
(\byear{2015}).
\bcomment{PMLR}
\end{bchapter}
\endbibitem

\bibitem[\protect\citeauthoryear{Schulman et~al.}{2017}]{schulman2017proximal}
\begin{botherref}
\oauthor{\bsnm{Schulman}, \binits{J.}},
\oauthor{\bsnm{Wolski}, \binits{F.}},
\oauthor{\bsnm{Dhariwal}, \binits{P.}},
\oauthor{\bsnm{Radford}, \binits{A.}},
\oauthor{\bsnm{Klimov}, \binits{O.}}:
Proximal policy optimization algorithms.
arXiv preprint arXiv:1707.06347
(2017)
\end{botherref}
\endbibitem

\bibitem[\protect\citeauthoryear{Kuba et~al.}{2022}]{kuba2021trust}
\begin{bchapter}
\bauthor{\bsnm{Kuba}, \binits{J.G.}},
\bauthor{\bsnm{Chen}, \binits{R.}},
\bauthor{\bsnm{Wen}, \binits{M.}},
\bauthor{\bsnm{Wen}, \binits{Y.}},
\bauthor{\bsnm{Sun}, \binits{F.}},
\bauthor{\bsnm{Wang}, \binits{J.}},
\bauthor{\bsnm{Yang}, \binits{Y.}}:
\bctitle{Trust region policy optimisation in multi-agent reinforcement learning}.
In: \bbtitle{International Conference on Learning Representations}
(\byear{2022}).
\burl{https://openreview.net/forum?id=EcGGFkNTxdJ}
\end{bchapter}
\endbibitem

\bibitem[\protect\citeauthoryear{Rashid et~al.}{2020}]{rashid2020monotonic}
\begin{barticle}
\bauthor{\bsnm{Rashid}, \binits{T.}},
\bauthor{\bsnm{Samvelyan}, \binits{M.}},
\bauthor{\bsnm{De~Witt}, \binits{C.S.}},
\bauthor{\bsnm{Farquhar}, \binits{G.}},
\bauthor{\bsnm{Foerster}, \binits{J.}},
\bauthor{\bsnm{Whiteson}, \binits{S.}}:
\batitle{Monotonic value function factorisation for deep multi-agent reinforcement learning}.
\bjtitle{The Journal of Machine Learning Research}
\bvolume{21}(\bissue{1}),
\bfpage{7234}--\blpage{7284}
(\byear{2020})
\end{barticle}
\endbibitem

\bibitem[\protect\citeauthoryear{Lowe et~al.}{2017}]{lowe2017multi}
\begin{botherref}
\oauthor{\bsnm{Lowe}, \binits{R.}},
\oauthor{\bsnm{Wu}, \binits{Y.I.}},
\oauthor{\bsnm{Tamar}, \binits{A.}},
\oauthor{\bsnm{Harb}, \binits{J.}},
\oauthor{\bsnm{Pieter~Abbeel}, \binits{O.}},
\oauthor{\bsnm{Mordatch}, \binits{I.}}:
Multi-agent actor-critic for mixed cooperative-competitive environments.
Advances in neural information processing systems
\textbf{30}
(2017)
\end{botherref}
\endbibitem

\bibitem[\protect\citeauthoryear{de~Witt et~al.}{2020}]{de2020independent}
\begin{botherref}
\oauthor{\bsnm{Witt}, \binits{C.S.}},
\oauthor{\bsnm{Gupta}, \binits{T.}},
\oauthor{\bsnm{Makoviichuk}, \binits{D.}},
\oauthor{\bsnm{Makoviychuk}, \binits{V.}},
\oauthor{\bsnm{Torr}, \binits{P.H.}},
\oauthor{\bsnm{Sun}, \binits{M.}},
\oauthor{\bsnm{Whiteson}, \binits{S.}}:
Is independent learning all you need in the starcraft multi-agent challenge?
arXiv preprint arXiv:2011.09533
(2020)
\end{botherref}
\endbibitem

\bibitem[\protect\citeauthoryear{Yu et~al.}{2022}]{yu2022surprising}
\begin{barticle}
\bauthor{\bsnm{Yu}, \binits{C.}},
\bauthor{\bsnm{Velu}, \binits{A.}},
\bauthor{\bsnm{Vinitsky}, \binits{E.}},
\bauthor{\bsnm{Gao}, \binits{J.}},
\bauthor{\bsnm{Wang}, \binits{Y.}},
\bauthor{\bsnm{Bayen}, \binits{A.}},
\bauthor{\bsnm{Wu}, \binits{Y.}}:
\batitle{The surprising effectiveness of ppo in cooperative multi-agent games}.
\bjtitle{Advances in Neural Information Processing Systems}
\bvolume{35},
\bfpage{24611}--\blpage{24624}
(\byear{2022})
\end{barticle}
\endbibitem

\bibitem[\protect\citeauthoryear{Daskalakis et~al.}{2020}]{daskalakis2020independent}
\begin{barticle}
\bauthor{\bsnm{Daskalakis}, \binits{C.}},
\bauthor{\bsnm{Foster}, \binits{D.J.}},
\bauthor{\bsnm{Golowich}, \binits{N.}}:
\batitle{Independent policy gradient methods for competitive reinforcement learning}.
\bjtitle{Advances in neural information processing systems}
\bvolume{33},
\bfpage{5527}--\blpage{5540}
(\byear{2020})
\end{barticle}
\endbibitem

\bibitem[\protect\citeauthoryear{Leonardos et~al.}{2022}]{leonardos2021global}
\begin{bchapter}
\bauthor{\bsnm{Leonardos}, \binits{S.}},
\bauthor{\bsnm{Overman}, \binits{W.}},
\bauthor{\bsnm{Panageas}, \binits{I.}},
\bauthor{\bsnm{Piliouras}, \binits{G.}}:
\bctitle{Global convergence of multi-agent policy gradient in markov potential games}.
In: \bbtitle{International Conference on Learning Representations}
(\byear{2022}).
\burl{https://openreview.net/forum?id=gfwON7rAm4}
\end{bchapter}
\endbibitem

\bibitem[\protect\citeauthoryear{Ding et~al.}{2022}]{ding2022independent}
\begin{bchapter}
\bauthor{\bsnm{Ding}, \binits{D.}},
\bauthor{\bsnm{Wei}, \binits{C.-Y.}},
\bauthor{\bsnm{Zhang}, \binits{K.}},
\bauthor{\bsnm{Jovanovic}, \binits{M.}}:
\bctitle{Independent policy gradient for large-scale markov potential games: Sharper rates, function approximation, and game-agnostic convergence}.
In: \bbtitle{International Conference on Machine Learning},
pp. \bfpage{5166}--\blpage{5220}
(\byear{2022}).
\bcomment{PMLR}
\end{bchapter}
\endbibitem

\bibitem[\protect\citeauthoryear{Foerster et~al.}{2018}]{foerster2018counterfactual}
\begin{bchapter}
\bauthor{\bsnm{Foerster}, \binits{J.}},
\bauthor{\bsnm{Farquhar}, \binits{G.}},
\bauthor{\bsnm{Afouras}, \binits{T.}},
\bauthor{\bsnm{Nardelli}, \binits{N.}},
\bauthor{\bsnm{Whiteson}, \binits{S.}}:
\bctitle{Counterfactual multi-agent policy gradients}.
In: \bbtitle{Proceedings of the AAAI Conference on Artificial Intelligence},
vol. \bseriesno{32}
(\byear{2018})
\end{bchapter}
\endbibitem

\bibitem[\protect\citeauthoryear{Engstrom et~al.}{2019}]{engstrom2019implementation}
\begin{bchapter}
\bauthor{\bsnm{Engstrom}, \binits{L.}},
\bauthor{\bsnm{Ilyas}, \binits{A.}},
\bauthor{\bsnm{Santurkar}, \binits{S.}},
\bauthor{\bsnm{Tsipras}, \binits{D.}},
\bauthor{\bsnm{Janoos}, \binits{F.}},
\bauthor{\bsnm{Rudolph}, \binits{L.}},
\bauthor{\bsnm{Madry}, \binits{A.}}:
\bctitle{Implementation matters in deep rl: A case study on ppo and trpo}.
In: \bbtitle{International Conference on Learning Representations}
(\byear{2019})
\end{bchapter}
\endbibitem

\bibitem[\protect\citeauthoryear{Hernandez-Leal et~al.}{2017}]{hernandez2017survey}
\begin{botherref}
\oauthor{\bsnm{Hernandez-Leal}, \binits{P.}},
\oauthor{\bsnm{Kaisers}, \binits{M.}},
\oauthor{\bsnm{Baarslag}, \binits{T.}},
\oauthor{\bsnm{De~Cote}, \binits{E.M.}}:
A survey of learning in multiagent environments: Dealing with non-stationarity.
arXiv preprint arXiv:1707.09183
(2017)
\end{botherref}
\endbibitem

\bibitem[\protect\citeauthoryear{Wen et~al.}{2022}]{wen2022multi}
\begin{barticle}
\bauthor{\bsnm{Wen}, \binits{M.}},
\bauthor{\bsnm{Kuba}, \binits{J.}},
\bauthor{\bsnm{Lin}, \binits{R.}},
\bauthor{\bsnm{Zhang}, \binits{W.}},
\bauthor{\bsnm{Wen}, \binits{Y.}},
\bauthor{\bsnm{Wang}, \binits{J.}},
\bauthor{\bsnm{Yang}, \binits{Y.}}:
\batitle{Multi-agent reinforcement learning is a sequence modeling problem}.
\bjtitle{Advances in Neural Information Processing Systems}
\bvolume{35},
\bfpage{16509}--\blpage{16521}
(\byear{2022})
\end{barticle}
\endbibitem

\bibitem[\protect\citeauthoryear{Zhao et~al.}{2023}]{zhao2023local}
\begin{bchapter}
\bauthor{\bsnm{Zhao}, \binits{Y.}},
\bauthor{\bsnm{Yang}, \binits{Z.}},
\bauthor{\bsnm{Wang}, \binits{Z.}},
\bauthor{\bsnm{Lee}, \binits{J.D.}}:
\bctitle{Local optimization achieves global optimality in multi-agent reinforcement learning}.
In: \bbtitle{Proceedings of the 40th International Conference on Machine Learning},
vol. \bseriesno{202},
pp. \bfpage{42200}--\blpage{42226}
(\byear{2023}).
\burl{https://proceedings.mlr.press/v202/zhao23j.html}
\end{bchapter}
\endbibitem

\bibitem[\protect\citeauthoryear{Li et~al.}{2022}]{li2021dealing}
\begin{bchapter}
\bauthor{\bsnm{Li}, \binits{W.}},
\bauthor{\bsnm{Wang}, \binits{X.}},
\bauthor{\bsnm{Jin}, \binits{B.}},
\bauthor{\bsnm{Sheng}, \binits{J.}},
\bauthor{\bsnm{Zha}, \binits{H.}}:
\bctitle{Dealing with non-stationarity in {MARL} via trust-region decomposition}.
In: \bbtitle{International Conference on Learning Representations}
(\byear{2022}).
\burl{https://openreview.net/forum?id=XHUxf5aRB3s}
\end{bchapter}
\endbibitem

\bibitem[\protect\citeauthoryear{Peng et~al.}{2021}]{peng2021facmac}
\begin{barticle}
\bauthor{\bsnm{Peng}, \binits{B.}},
\bauthor{\bsnm{Rashid}, \binits{T.}},
\bauthor{\bsnm{Witt}, \binits{C.}},
\bauthor{\bsnm{Kamienny}, \binits{P.-A.}},
\bauthor{\bsnm{Torr}, \binits{P.}},
\bauthor{\bsnm{B{\"o}hmer}, \binits{W.}},
\bauthor{\bsnm{Whiteson}, \binits{S.}}:
\batitle{Facmac: Factored multi-agent centralised policy gradients}.
\bjtitle{Advances in Neural Information Processing Systems}
\bvolume{34},
\bfpage{12208}--\blpage{12221}
(\byear{2021})
\end{barticle}
\endbibitem

\bibitem[\protect\citeauthoryear{Todorov et~al.}{2012}]{todorov2012mujoco}
\begin{bchapter}
\bauthor{\bsnm{Todorov}, \binits{E.}},
\bauthor{\bsnm{Erez}, \binits{T.}},
\bauthor{\bsnm{Tassa}, \binits{Y.}}:
\bctitle{Mujoco: A physics engine for model-based control}.
In: \bbtitle{2012 IEEE/RSJ International Conference on Intelligent Robots and Systems},
pp. \bfpage{5026}--\blpage{5033}
(\byear{2012}).
\bcomment{IEEE}
\end{bchapter}
\endbibitem

\bibitem[\protect\citeauthoryear{Wu et~al.}{2021}]{wu2021coordinated}
\begin{barticle}
\bauthor{\bsnm{Wu}, \binits{Z.}},
\bauthor{\bsnm{Yu}, \binits{C.}},
\bauthor{\bsnm{Ye}, \binits{D.}},
\bauthor{\bsnm{Zhang}, \binits{J.}},
\bauthor{\bsnm{Zhuo}, \binits{H.H.}}, \betal:
\batitle{Coordinated proximal policy optimization}.
\bjtitle{Advances in Neural Information Processing Systems}
\bvolume{34},
\bfpage{26437}--\blpage{26448}
(\byear{2021})
\end{barticle}
\endbibitem

\bibitem[\protect\citeauthoryear{Liu et~al.}{2019}]{liu2019neural}
\begin{botherref}
\oauthor{\bsnm{Liu}, \binits{B.}},
\oauthor{\bsnm{Cai}, \binits{Q.}},
\oauthor{\bsnm{Yang}, \binits{Z.}},
\oauthor{\bsnm{Wang}, \binits{Z.}}:
Neural trust region/proximal policy optimization attains globally optimal policy.
Advances in neural information processing systems
\textbf{32}
(2019)
\end{botherref}
\endbibitem

\bibitem[\protect\citeauthoryear{Huang et~al.}{2021}]{huang2021neural}
\begin{botherref}
\oauthor{\bsnm{Huang}, \binits{N.-C.}},
\oauthor{\bsnm{Hsieh}, \binits{P.-C.}},
\oauthor{\bsnm{Ho}, \binits{K.-H.}},
\oauthor{\bsnm{Yao}, \binits{H.-Y.}},
\oauthor{\bsnm{Hu}, \binits{K.-C.}},
\oauthor{\bsnm{Ouyang}, \binits{L.-C.}},
\oauthor{\bsnm{Wu}, \binits{I.}}, et al.:
Neural ppo-clip attains global optimality: A hinge loss perspective.
arXiv preprint arXiv:2110.13799
(2021)
\end{botherref}
\endbibitem

\bibitem[\protect\citeauthoryear{Kakade and Langford}{2002}]{kakade2002approximately}
\begin{bchapter}
\bauthor{\bsnm{Kakade}, \binits{S.}},
\bauthor{\bsnm{Langford}, \binits{J.}}:
\bctitle{Approximately optimal approximate reinforcement learning}.
In: \bbtitle{Proceedings of the Nineteenth International Conference on Machine Learning},
pp. \bfpage{267}--\blpage{274}
(\byear{2002})
\end{bchapter}
\endbibitem

\bibitem[\protect\citeauthoryear{Shani et~al.}{2020}]{shani2020adaptive}
\begin{bchapter}
\bauthor{\bsnm{Shani}, \binits{L.}},
\bauthor{\bsnm{Efroni}, \binits{Y.}},
\bauthor{\bsnm{Mannor}, \binits{S.}}:
\bctitle{Adaptive trust region policy optimization: Global convergence and faster rates for regularized mdps}.
In: \bbtitle{Proceedings of the AAAI Conference on Artificial Intelligence},
vol. \bseriesno{34},
pp. \bfpage{5668}--\blpage{5675}
(\byear{2020})
\end{bchapter}
\endbibitem

\bibitem[\protect\citeauthoryear{Papoudakis et~al.}{2020}]{papoudakis2020benchmarking}
\begin{botherref}
\oauthor{\bsnm{Papoudakis}, \binits{G.}},
\oauthor{\bsnm{Christianos}, \binits{F.}},
\oauthor{\bsnm{Sch{\"a}fer}, \binits{L.}},
\oauthor{\bsnm{Albrecht}, \binits{S.V.}}:
Benchmarking multi-agent deep reinforcement learning algorithms in cooperative tasks.
arXiv preprint arXiv:2006.07869
(2020)
\end{botherref}
\endbibitem

\bibitem[\protect\citeauthoryear{Christianos et~al.}{2023}]{christianos2022pareto}
\begin{botherref}
\oauthor{\bsnm{Christianos}, \binits{F.}},
\oauthor{\bsnm{Papoudakis}, \binits{G.}},
\oauthor{\bsnm{Albrecht}, \binits{S.V.}}:
Pareto actor-critic for equilibrium selection in multi-agent reinforcement learning.
Transactions on Machine Learning Research
(2023)
\end{botherref}
\endbibitem

\bibitem[\protect\citeauthoryear{de~Witt et~al.}{2020}]{de2020deep}
\begin{botherref}
\oauthor{\bsnm{Witt}, \binits{C.S.}},
\oauthor{\bsnm{Peng}, \binits{B.}},
\oauthor{\bsnm{Kamienny}, \binits{P.-A.}},
\oauthor{\bsnm{Torr}, \binits{P.}},
\oauthor{\bsnm{B{\"o}hmer}, \binits{W.}},
\oauthor{\bsnm{Whiteson}, \binits{S.}}:
Deep multi-agent reinforcement learning for decentralized continuous cooperative control.
arXiv preprint arXiv:2003.06709
\textbf{19}
(2020)
\end{botherref}
\endbibitem

\bibitem[\protect\citeauthoryear{Samvelyan et~al.}{2019}]{samvelyan2019starcraft}
\begin{botherref}
\oauthor{\bsnm{Samvelyan}, \binits{M.}},
\oauthor{\bsnm{Rashid}, \binits{T.}},
\oauthor{\bsnm{Witt}, \binits{C.S.}},
\oauthor{\bsnm{Farquhar}, \binits{G.}},
\oauthor{\bsnm{Nardelli}, \binits{N.}},
\oauthor{\bsnm{Rudner}, \binits{T.G.J.}},
\oauthor{\bsnm{Hung}, \binits{C.-M.}},
\oauthor{\bsnm{Torr}, \binits{P.H.S.}},
\oauthor{\bsnm{Foerster}, \binits{J.}},
\oauthor{\bsnm{Whiteson}, \binits{S.}}:
{The} {StarCraft} {Multi}-{Agent} {Challenge}.
CoRR
\textbf{abs/1902.04043}
(2019)
\end{botherref}
\endbibitem

\bibitem[\protect\citeauthoryear{Zhong et~al.}{2024}]{zhong2023heterogeneous}
\begin{barticle}
\bauthor{\bsnm{Zhong}, \binits{Y.}},
\bauthor{\bsnm{Kuba}, \binits{J.G.}},
\bauthor{\bsnm{Feng}, \binits{X.}},
\bauthor{\bsnm{Hu}, \binits{S.}},
\bauthor{\bsnm{Ji}, \binits{J.}},
\bauthor{\bsnm{Yang}, \binits{Y.}}:
\batitle{Heterogeneous-agent reinforcement learning}.
\bjtitle{Journal of Machine Learning Research}
\bvolume{25}(\bissue{32}),
\bfpage{1}--\blpage{67}
(\byear{2024})
\end{barticle}
\endbibitem

\bibitem[\protect\citeauthoryear{Kirkpatrick et~al.}{2017}]{kirkpatrick2017overcoming}
\begin{barticle}
\bauthor{\bsnm{Kirkpatrick}, \binits{J.}},
\bauthor{\bsnm{Pascanu}, \binits{R.}},
\bauthor{\bsnm{Rabinowitz}, \binits{N.}},
\bauthor{\bsnm{Veness}, \binits{J.}},
\bauthor{\bsnm{Desjardins}, \binits{G.}},
\bauthor{\bsnm{Rusu}, \binits{A.A.}},
\bauthor{\bsnm{Milan}, \binits{K.}},
\bauthor{\bsnm{Quan}, \binits{J.}},
\bauthor{\bsnm{Ramalho}, \binits{T.}},
\bauthor{\bsnm{Grabska-Barwinska}, \binits{A.}}, \betal:
\batitle{Overcoming catastrophic forgetting in neural networks}.
\bjtitle{Proceedings of the national academy of sciences}
\bvolume{114}(\bissue{13}),
\bfpage{3521}--\blpage{3526}
(\byear{2017})
\end{barticle}
\endbibitem

\bibitem[\protect\citeauthoryear{Li and He}{2023}]{li2023multiagent}
\begin{botherref}
\oauthor{\bsnm{Li}, \binits{H.}},
\oauthor{\bsnm{He}, \binits{H.}}:
Multiagent trust region policy optimization.
IEEE Transactions on Neural Networks and Learning Systems
(2023)
\end{botherref}
\endbibitem

\bibitem[\protect\citeauthoryear{Gu et~al.}{2021}]{gu2021multi}
\begin{botherref}
\oauthor{\bsnm{Gu}, \binits{S.}},
\oauthor{\bsnm{Kuba}, \binits{J.G.}},
\oauthor{\bsnm{Wen}, \binits{M.}},
\oauthor{\bsnm{Chen}, \binits{R.}},
\oauthor{\bsnm{Wang}, \binits{Z.}},
\oauthor{\bsnm{Tian}, \binits{Z.}},
\oauthor{\bsnm{Wang}, \binits{J.}},
\oauthor{\bsnm{Knoll}, \binits{A.}},
\oauthor{\bsnm{Yang}, \binits{Y.}}:
Multi-agent constrained policy optimisation.
arXiv preprint arXiv:2110.02793
(2021)
\end{botherref}
\endbibitem

\bibitem[\protect\citeauthoryear{Tucker et~al.}{2018}]{tucker2018mirage}
\begin{bchapter}
\bauthor{\bsnm{Tucker}, \binits{G.}},
\bauthor{\bsnm{Bhupatiraju}, \binits{S.}},
\bauthor{\bsnm{Gu}, \binits{S.}},
\bauthor{\bsnm{Turner}, \binits{R.}},
\bauthor{\bsnm{Ghahramani}, \binits{Z.}},
\bauthor{\bsnm{Levine}, \binits{S.}}:
\bctitle{The mirage of action-dependent baselines in reinforcement learning}.
In: \bbtitle{International Conference on Machine Learning},
pp. \bfpage{5015}--\blpage{5024}
(\byear{2018}).
\bcomment{PMLR}
\end{bchapter}
\endbibitem

\bibitem[\protect\citeauthoryear{Chung et~al.}{2021}]{chung2021beyond}
\begin{bchapter}
\bauthor{\bsnm{Chung}, \binits{W.}},
\bauthor{\bsnm{Thomas}, \binits{V.}},
\bauthor{\bsnm{Machado}, \binits{M.C.}},
\bauthor{\bsnm{Le~Roux}, \binits{N.}}:
\bctitle{Beyond variance reduction: Understanding the true impact of baselines on policy optimization}.
In: \bbtitle{International Conference on Machine Learning},
pp. \bfpage{1999}--\blpage{2009}
(\byear{2021}).
\bcomment{PMLR}
\end{bchapter}
\endbibitem

\bibitem[\protect\citeauthoryear{Munos et~al.}{2016}]{munos2016safe}
\begin{botherref}
\oauthor{\bsnm{Munos}, \binits{R.}},
\oauthor{\bsnm{Stepleton}, \binits{T.}},
\oauthor{\bsnm{Harutyunyan}, \binits{A.}},
\oauthor{\bsnm{Bellemare}, \binits{M.}}:
Safe and efficient off-policy reinforcement learning.
Advances in neural information processing systems
\textbf{29}
(2016)
\end{botherref}
\endbibitem

\bibitem[\protect\citeauthoryear{Kuba et~al.}{2021}]{kuba2021settling}
\begin{barticle}
\bauthor{\bsnm{Kuba}, \binits{J.G.}},
\bauthor{\bsnm{Wen}, \binits{M.}},
\bauthor{\bsnm{Meng}, \binits{L.}},
\bauthor{\bsnm{Zhang}, \binits{H.}},
\bauthor{\bsnm{Mguni}, \binits{D.}},
\bauthor{\bsnm{Wang}, \binits{J.}},
\bauthor{\bsnm{Yang}, \binits{Y.}}, \betal:
\batitle{Settling the variance of multi-agent policy gradients}.
\bjtitle{Advances in Neural Information Processing Systems}
\bvolume{34},
\bfpage{13458}--\blpage{13470}
(\byear{2021})
\end{barticle}
\endbibitem

\bibitem[\protect\citeauthoryear{Wang et~al.}{2023}]{wang2023order}
\begin{bchapter}
\bauthor{\bsnm{Wang}, \binits{X.}},
\bauthor{\bsnm{Tian}, \binits{Z.}},
\bauthor{\bsnm{Wan}, \binits{Z.}},
\bauthor{\bsnm{Wen}, \binits{Y.}},
\bauthor{\bsnm{Wang}, \binits{J.}},
\bauthor{\bsnm{Zhang}, \binits{W.}}:
\bctitle{Order matters: Agent-by-agent policy optimization}.
In: \bbtitle{The Eleventh International Conference on Learning Representations}
(\byear{2023}).
\burl{https://openreview.net/forum?id=Q-neeWNVv1}
\end{bchapter}
\endbibitem

\bibitem[\protect\citeauthoryear{Sunehag et~al.}{2018}]{sunehag2017value}
\begin{bchapter}
\bauthor{\bsnm{Sunehag}, \binits{P.}},
\bauthor{\bsnm{Lever}, \binits{G.}},
\bauthor{\bsnm{Gruslys}, \binits{A.}},
\bauthor{\bsnm{Czarnecki}, \binits{W.M.}},
\bauthor{\bsnm{Zambaldi}, \binits{V.}},
\bauthor{\bsnm{Jaderberg}, \binits{M.}},
\bauthor{\bsnm{Lanctot}, \binits{M.}},
\bauthor{\bsnm{Sonnerat}, \binits{N.}},
\bauthor{\bsnm{Leibo}, \binits{J.Z.}},
\bauthor{\bsnm{Tuyls}, \binits{K.}},
\bauthor{\bsnm{Graepel}, \binits{T.}}:
\bctitle{Value-decomposition networks for cooperative multi-agent learning based on team reward}.
In: \bbtitle{Proceedings of the 17th International Conference on Autonomous Agents and MultiAgent Systems}.
\bsertitle{AAMAS '18},
pp. \bfpage{2085}--\blpage{2087}.
\bpublisher{International Conference on Autonomous Agents and Multiagent Systems},
\blocation{Richland, SC}
(\byear{2018})
\end{bchapter}
\endbibitem

\bibitem[\protect\citeauthoryear{Neu et~al.}{2017}]{neu2017unified}
\begin{botherref}
\oauthor{\bsnm{Neu}, \binits{G.}},
\oauthor{\bsnm{Jonsson}, \binits{A.}},
\oauthor{\bsnm{Gómez}, \binits{V.}}:
A unified view of entropy-regularized markov decision processes.
CoRR
\textbf{abs/1705.07798}
(2017)
\end{botherref}
\endbibitem

\bibitem[\protect\citeauthoryear{Lan}{2023}]{lan2023policy}
\begin{barticle}
\bauthor{\bsnm{Lan}, \binits{G.}}:
\batitle{Policy mirror descent for reinforcement learning: Linear convergence, new sampling complexity, and generalized problem classes}.
\bjtitle{Mathematical programming}
\bvolume{198}(\bissue{1}),
\bfpage{1059}--\blpage{1106}
(\byear{2023})
\end{barticle}
\endbibitem

\bibitem[\protect\citeauthoryear{Beck}{2017}]{beck2017first}
\begin{bbook}
\bauthor{\bsnm{Beck}, \binits{A.}}:
\bbtitle{First-order Methods in Optimization}.
\bpublisher{SIAM-Society for Industrial and Applied Mathematics},
\blocation{Philadelphia, PA, USA}
(\byear{2017})
\end{bbook}
\endbibitem

\bibitem[\protect\citeauthoryear{Agarwal et~al.}{2021}]{agarwal2021theory}
\begin{barticle}
\bauthor{\bsnm{Agarwal}, \binits{A.}},
\bauthor{\bsnm{Kakade}, \binits{S.M.}},
\bauthor{\bsnm{Lee}, \binits{J.D.}},
\bauthor{\bsnm{Mahajan}, \binits{G.}}:
\batitle{On the theory of policy gradient methods: Optimality, approximation, and distribution shift}.
\bjtitle{The Journal of Machine Learning Research}
\bvolume{22}(\bissue{1}),
\bfpage{4431}--\blpage{4506}
(\byear{2021})
\end{barticle}
\endbibitem

\bibitem[\protect\citeauthoryear{Schulman et~al.}{2017}]{schulman2017equivalence}
\begin{botherref}
\oauthor{\bsnm{Schulman}, \binits{J.}},
\oauthor{\bsnm{Chen}, \binits{X.}},
\oauthor{\bsnm{Abbeel}, \binits{P.}}:
Equivalence between policy gradients and soft q-learning.
arXiv preprint arXiv:1704.06440
(2017)
\end{botherref}
\endbibitem

\end{thebibliography}



\end{document}